\documentclass[10pt]{article} 
\usepackage[accepted]{tmlr}

\usepackage{amsmath,amsfonts,bm}

\def\eqref#1{equation~\ref{#1}}

\def\ceil#1{\lceil #1 \rceil}

\def\1{\bm{1}}

\def\vd{{\bm{d}}}

\def\vg{{\bm{g}}}

\def\vo{{\bm{o}}}

\def\vv{{\bm{v}}}

\def\mI{{\bm{I}}}

\DeclareMathAlphabet{\mathsfit}{\encodingdefault}{\sfdefault}{m}{sl}
\SetMathAlphabet{\mathsfit}{bold}{\encodingdefault}{\sfdefault}{bx}{n}

\def\sA{{\mathbb{A}}}

\def\sC{{\mathbb{C}}}
\def\sD{{\mathbb{D}}}
\def\sF{{\mathbb{F}}}

\def\sM{{\mathbb{M}}}

\def\sS{{\mathbb{S}}}
\def\sT{{\mathbb{T}}}

\DeclareMathOperator*{\argmin}{arg\,min}

\usepackage[margin=1in]{geometry}
\usepackage[utf8]{inputenc} 
\usepackage[T1]{fontenc}    
\usepackage{hyperref}       
\usepackage{url}            
\usepackage{booktabs}       
\usepackage{amsfonts}       
\usepackage{nicefrac}       
\usepackage{lipsum}
\usepackage{microtype}      
\usepackage{multirow}
\usepackage{xcolor}         
\usepackage{algorithm}
\usepackage[title]{appendix}
\usepackage{amssymb}
\usepackage{csquotes}
\usepackage{bbm}
\usepackage{amsthm}
\usepackage{graphicx} 
\usepackage{tablefootnote}
\usepackage{bm}
\usepackage{subcaption}

\let\classAND\AND
\let\AND\relax
\usepackage{algorithmic}

\let\AND\classAND
\AtBeginEnvironment{algorithmic}{\let\AND\algoAND}

\floatname{algorithm}{Algorithm}

\newenvironment{hproof}{%
  \proof}{\endproof}

\usepackage{mathtools}
\usepackage{multirow}
\usepackage{adjustbox}
\usepackage{threeparttable}

\usepackage{wrapfig}

\usepackage{parskip}

\usepackage{thm-restate}
\usepackage{dsfont}

\usepackage{array}
\newcolumntype{P}[1]{>{\centering\arraybackslash}p{#1}}

\newcommand{\norm}[1]{\left\lVert#1\right\rVert}

\newtheorem{theorem}{Theorem}[section]

\newtheorem{lemma}[theorem]{Lemma}
\newtheorem{remark}{Remark}
\newtheorem{definition}{Definition}

\usepackage{hyperref}
\usepackage{url}

\title{
Individual Privacy Accounting for \\ Differentially Private Stochastic Gradient Descent
}

\makeatletter
\def\@fnsymbol#1{\ensuremath{\ifcase#1\or \dagger\or \text{*} \or \ddagger\or
   \mathsection\or  \mathparagraph\or \|\or **\or \dagger\dagger
   \or \ddagger\ddagger \else\@ctrerr\fi}}
\makeatother

\author{\name Da Yu \email yuda3@mail2.sysu.edu.cn \\
 \addr Sun Yat-sen University 
 \AND
 \name Gautam Kamath\thanks{Supported by an NSERC Discovery Grant, an unrestricted gift from Google, and a University of Waterloo startup grant.} \textsuperscript{\space*} \email g@csail.mit.edu \\
 \addr Cheriton School of Computer Science \\
 University of Waterloo
 \AND
 \name Janardhan Kulkarni\textsuperscript{\space*} \email jakul@microsoft.com \\
 \addr Microsoft Research
 \AND
 \name Tie-Yan Liu\textsuperscript{\space*} \email tyliu@microsoft.com \\
 \addr Microsoft Research
 \AND
 \name Jian Yin\textsuperscript{\space*} \email issjyin@mail.sysu.edu.cn \\
 \addr Sun Yat-sen University
 \AND
 \name Huishuai Zhang\textsuperscript{\space}\thanks{Authors are listed in alphabetical order.} \email huzhang@microsoft.com \\
 \addr Microsoft Research
}

\def\month{09}  
\def\year{2023} 
\def\openreview{\url{https://openreview.net/forum?id=l4Jcxs0fpC}} 

\begin{document}

\maketitle

\begin{abstract}
Differentially private stochastic gradient descent (DP-SGD) is the workhorse algorithm for recent advances in private deep learning. 
It provides a single privacy guarantee to all datapoints in the dataset.
We propose \emph{output-specific} $(\varepsilon,\delta)$-DP to characterize privacy guarantees for individual examples when releasing models trained by DP-SGD.
We also design an efficient algorithm to investigate individual privacy  across a number of datasets.
We find that most examples enjoy stronger privacy guarantees than the worst-case bound. 
We further discover that the training loss and the privacy parameter of an example are well-correlated.
This implies groups that are underserved in terms of model utility simultaneously experience weaker privacy guarantees. 
For example, on CIFAR-10, the average $\varepsilon$ of the class with the lowest test accuracy is 44.2\% higher than that of the class with the highest accuracy. Our code is available at \url{https://github.com/dayu11/individual_privacy_of_DPSGD}.
\end{abstract}

\section{Introduction}
\label{sec:intro}

Differential privacy is a strong notion of data privacy, enabling rich forms of privacy-preserving data analysis \citep{DworkMNS06,DworkR14}.
Informally speaking, it quantitatively bounds the maximum influence of any datapoint using a privacy parameter $\varepsilon$, where smaller values of $\varepsilon$ correspond to stronger privacy guarantees. Training deep models with differential privacy is an active research area \citep{PapernotAEGT17,ZhuYCW20,AnilGGKM21,YuNBGIKKLMWYZ22,LiTLH22,GolatkarAWRKS22,MehtaTKC22,DeBHSB22,bu2022scalable,MehtaKTKC22}. Models trained with differential privacy not only provide theoretical privacy guarantees to their data owners but also are more robust against empirical attacks \citep{RahmanRLMW18,BernauGRK19,CarliniLEKS19,JagielskiUO20,NasrSTPC21}.

Differentially private stochastic gradient descent (DP-SGD) is the most popular algorithm for differentially private deep learning \citep{SongCS13,BassilyST14,AbadiCGMMTZ16}. 
At each step, DP-SGD takes the model from the previous step and the dataset as inputs. It adds isotropic Gaussian noise to the average gradient of the current step.  Models trained with DP-SGD satisfy $(\varepsilon,\delta)$-differential privacy.
The canonical notion of differential privacy, including $(\varepsilon,\delta)$-DP,  considers the {\em worst-case} privacy over all possible inputs.
In the case of DP-SGD, this results in the privacy cost of all examples being computed with the largest possible magnitude of individual gradients, i.e., the gradient clipping threshold.

In practice, we may care more about the privacy guarantees of the models that will be deployed, which depend on the observed training trajectories. 
Broadly speaking, different examples may have very different impacts on a trained model \citep{FeldmanZ20, JiangZTM21}. 
Some examples may be easier to learn and hence the magnitudes of their individual gradients along the observed training trajectory could be much smaller than the clipping threshold.
Such a fine-grained privacy guarantee can not be inferred by the canonical $(\varepsilon,\delta)$-DP because it requires the privacy guarantee to hold for all possible datasets and training trajectories.
In this paper, we define \emph{output-specific} $(\varepsilon,\delta)$-DP, which adapts to the training trajectory of the model to analyze the individual privacy of DP-SGD. 
Our definition captures the impact of various factors, such as the training set and algorithmic randomness, on individual privacy.
We also develop algorithms to efficiently and accurately estimate individual privacy.  

It turns out that, unsurprisingly, for common benchmarks, many examples experience much stronger privacy guarantees than implied by the worst-case DP analysis. 
To illustrate this, we plot the individual privacy parameters of four benchmark datasets in Figure~\ref{fig:eps_teaser}. 
Experimental details, as well as more results, are in Sections~\ref{sec:exp} and~\ref{sec:fairness}.
To the best of our knowledge, this paper is the first to reveal the disparity in individual privacy when running DP-SGD.

 \begin{figure*}
    \centering
  \includegraphics[width=1.0\linewidth]{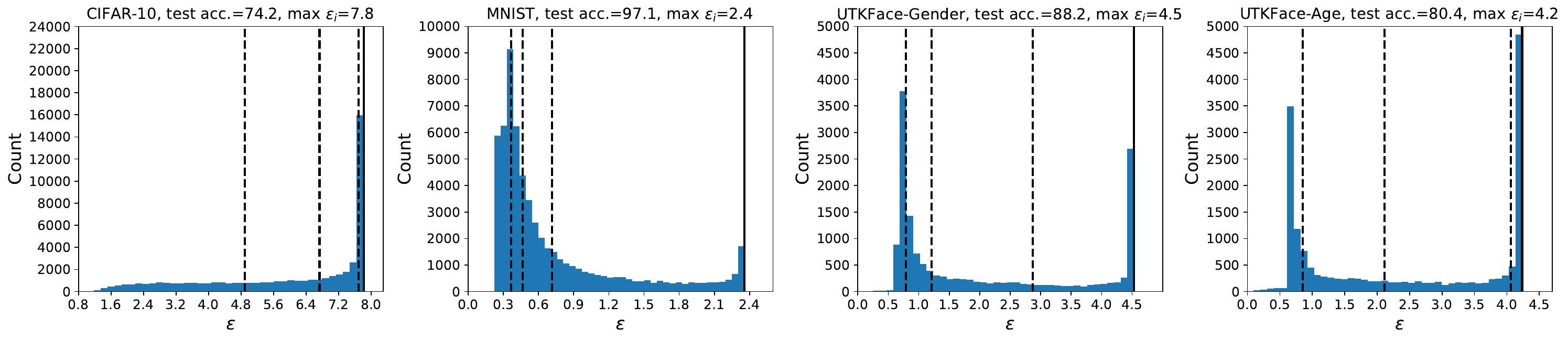}
  \caption{Individual privacy parameters of models trained by DP-SGD. The value of $\delta$ is $1\times 10^{-5}$. The dashed lines indicate $30\%$, $50\%$, and $70\%$ of datapoints. The black solid line shows the worst-case privacy parameter. }
  \label{fig:eps_teaser}
\end{figure*}

Further,  we demonstrate a strong correlation between the privacy parameter of an example and its final training loss.
That is, the examples with higher training loss also have higher privacy parameters in general.
This suggests that the examples that suffer unfairness in terms of worse privacy are also the ones that have worse utility.
See Figure~\ref{fig:eps_loss_corr_class_cifar} for an illustration. 
While prior works have shown that underrepresented groups experience worse utility~\citep{BuolamwiniG18}, and that these disparities are amplified when models are trained privately~\citep{BagdasaryanPS19,hansen2022impact,noe2022exploring,LowyGR22}, we are the first to show that the privacy guarantee \emph{and} utility are negatively impacted concurrently. 
In contrast,  prior work that takes a worst-case perspective for privacy accounting, results in a uniform privacy guarantee for all training examples. 
For instance, when running gender classification on  UTKFace, the average $\varepsilon$ of the race with the lowest test accuracy is 35.1\% higher than that of the race with the highest accuracy.

\label{sec:pre}

 \begin{figure}
    \centering
  \includegraphics[width=0.75\linewidth]{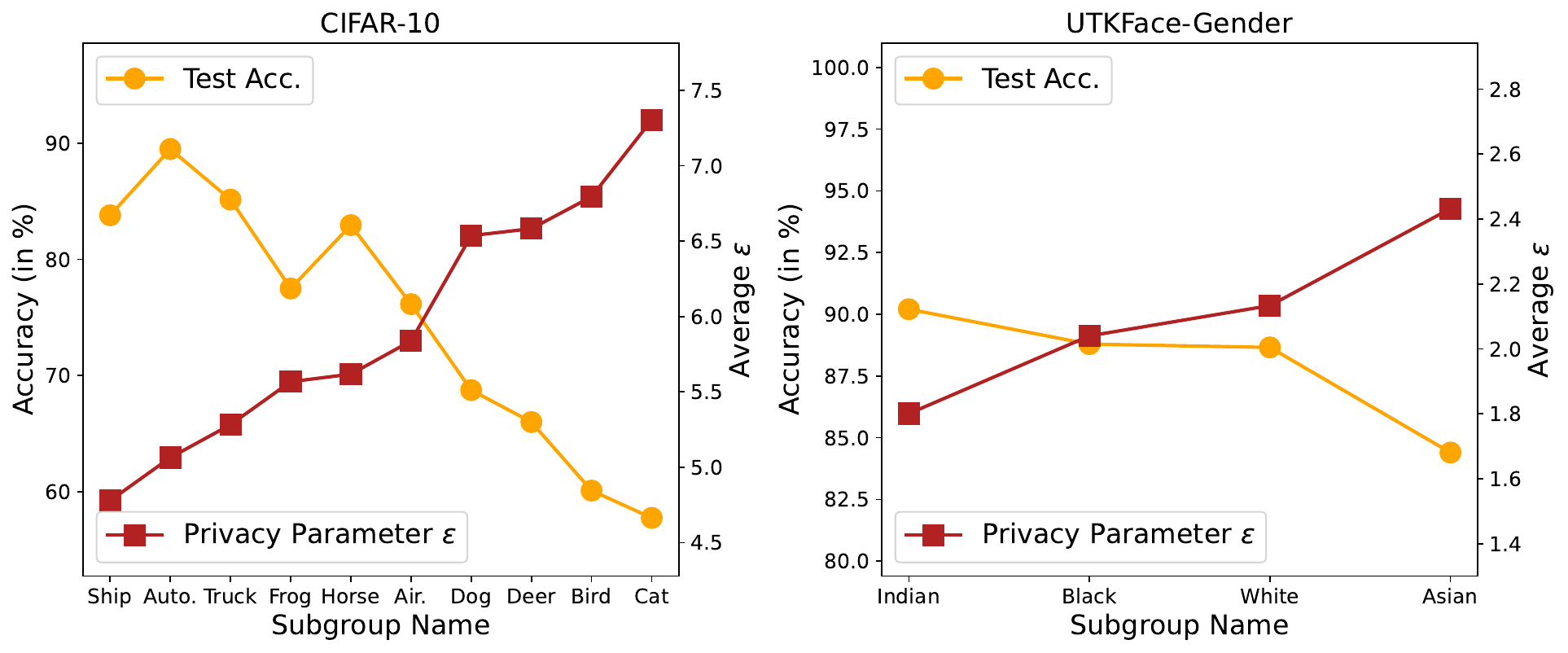}
  \caption{Accuracy and average $\varepsilon$ of different groups on CIFAR-10 and UTK-Face. Groups with worse accuracy also  have worse privacy in general. }
  \label{fig:eps_loss_corr_class_cifar}
\end{figure}

\subsection{Related Work}
\label{sec:related}

Several works have explored individual privacy analysis in differentially private learning. \citet{jorgensen2015conservative,MuhlB22}, and the work subsequent to ours of  \citet{KoskelaTH22,BoenischMDRP23}, design learning algorithms that satisfy prespecified individual privacy parameters. Those prespecified parameters are independent of the learning algorithm, e.g., in some applications different users may have different expectations of privacy. \citet{Wang19} defines \emph{Per-instance differential privacy} to analyze individual privacy when the target example is put in a fixed dataset.
\citet{RedbergW21} investigate per-instance DP of the objective perturbation algorithm \citep{KiferST12}.
\citet{GolatkarAWRKS22} explore  per-instance DP of differentially private batch gradient descent.   
\citet{RedbergW21,GolatkarAWRKS22} focus on (strongly) convex objective functions because otherwise fixing the dataset is not sufficient to determine the individual privacy parameters.
In this work, we define output-specific $(\varepsilon, \delta)$-DP that allows us to study the individual privacy of non-convex models trained by DP-SGD. We also conduct experiments on several datasets and demonstrate a strong correlation between privacy and utility.

 \citet{FeldmanZ21} design individual \emph{privacy filters} to make use of the variation in individual sensitivity. The filters allow examples with smaller per-step privacy costs to run for more steps until the accumulated cost reaches a target budget. Additionally, \citet{FeldmanZ21} study the individual privacy of DP-GD and demonstrate that examples often experience stronger privacy than worst-case analysis suggests. However, computing individual privacy requires calculating gradient norms for all training examples at every step.  In this work, we present an efficient algorithm for estimating the individual privacy of DP-SGD. Our algorithm accurately estimates individual privacy while only occasionally computing the gradient norms of all examples.

Our  output-specific $(\varepsilon, \delta)$-DP and the notion of \emph{ex-post} DP by \citet{ligett2017accuracy}  both tailor the privacy guarantee to algorithm outcomes. Ex-post DP bounds the ratio between two probability/density functions at a single outcome. It can be generalized to pure differential privacy ($(\varepsilon,0)$-DP). In contrast, DP-SGD uses Gaussian mechanisms and provides approximate differential privacy ($(\varepsilon, \delta)$-DP). There is no clean conversion between $(\varepsilon, \delta)$-DP and ex-post DP \citep{Meiser18}. Therefore, our notion is necessary for analyzing individual privacy within the $(\varepsilon, \delta)$-DP framework.

\section{Preliminaries}

We first give some background on DP-SGD and explain why the canonical $(\varepsilon, \delta)$-DP is not suitable for  measuring individual privacy. Then we define output-specific $(\varepsilon, \delta)$-DP. Finally, we give empirical evidence showing that providing the same privacy to all examples is not ideal.

\subsection{Background on Differentially Private SGD}
\label{subsec:background}

The privacy guarantee of DP-SGD is measured by $(\varepsilon,\delta)$-differential privacy.
\begin{definition}
\label{def:definition_dp}[$(\varepsilon,\delta)$-DP]
An algorithm $\mathcal{A}:\mathcal{D}\rightarrow \mathcal{O}$  satisfies $(\varepsilon,\delta)$-differential privacy if for any pair of neighboring datasets $\sD, \sD'\in \mathcal{D}$ and any subset of outputs $\sS\subset \mathcal{O}$ it holds that 
\[\Pr[\mathcal{A}(\sD)\in \sS]\leq e^{\varepsilon}\Pr[\mathcal{A}(\sD')\in \sS]+\delta.\] 
\end{definition}

Two datasets $\sD, \sD'$ are neighboring datasets if they only differ in one datapoint.  DP-SGD uses Rényi differential privacy (RDP) \citep{Mironov17} in privacy accounting to get a tighter composition bound \citep{AbadiCGMMTZ16}. After training, the accumulated RDP is converted to $(\varepsilon,\delta)$-DP.
RDP measures the Rényi divergence at different orders. The Rényi divergence between two probability distributions $\mu$ and $\nu$ at order $\alpha$ is 
\[D_{\alpha}(\mu||\nu)=\frac{1}{\alpha-1} \log\int (\frac{d\mu}{d\nu})^{\alpha}d\nu.\]
Let $D^{\leftrightarrow}_{\alpha}(\mu||\nu)=\max (D_{\alpha}(\mu||\nu),D_{\alpha}(\nu||\mu))$ be the maximum divergence of two directions.  The definition of RDP  is as follows.
\begin{definition}
\label{def:rdp}[Rényi differential privacy \citep{Mironov17}]
A randomized algorithm $\mathcal{A}: \mathcal{D}\rightarrow \mathcal{O}$  satisfies $(\alpha,\rho)$-RDP if for any neighboring datasets $\sD, \sD'\in \mathcal{D}$ it holds that
\[D^{\leftrightarrow}_{\alpha}(\mathcal{A}(\sD) || \mathcal{A}(\sD'))\leq \rho.\]
\end{definition}
When $\mathcal{A}$ is a deep learning algorithm, it is infeasible to directly measure the output distributions because of the non-convex nature of neural networks. To address this, DP-SGD makes each gradient update differentially private and uses the composition property of differential privacy  to reason about the overall privacy cost.
\begin{definition}
\label{def:rdp_composition}[Composition of RDP \citep{Mironov17}]
Let $\mathcal{A}_{1}:\mathcal{D}\rightarrow \mathcal{O}_{1}$ be $(\alpha,\rho_{1})$-RDP and $\mathcal{A}_{2}: \mathcal{O}_{1}\times \mathcal{D}\rightarrow \mathcal{O}_{2}$ be $(\alpha,\rho_{2})$-RDP, then the mechanism defined as $(X,Y)$, where $X\sim \mathcal{A}_{1}(\sD)$ and $Y\sim \mathcal{A}_{2}(X,\sD)$, satisfies $(\alpha,\rho_{1}+\rho_{2})$-RDP.
\end{definition}

The output of the composed algorithm is a tuple containing the outputs from all steps. Consequently, the output of DP-SGD at step $T$ is a sequence of models $(\theta_{1},\theta_{2},\ldots,\theta_{T})$. 

The privacy cost of a target example depends on its gradients along the training trajectory. We formalize the output distributions of DP-SGD at each step to illustrate this. DP-SGD uses Poisson sampling, i.e., each example is sampled independently with probability $p$. Let $\vv=\sum_{i\in \sM}\vg_{i}$ be the sum of the minibatch gradients of $\sD$, where $\sM$ is the set of sampled indices. Consider also a neighboring dataset $\sD'$ that has one datapoint $\vd$ (with gradient $\vg$) added. Because of Poisson sampling, the output is exactly $\vv$ with probability $1-p$ ($\vg$ is not sampled) and is $\vv' = \vv + \vg$ with probability $p$ ($\vg$ is sampled). After adding isotropic Gaussian noise,  the output distributions of two neighboring datasets are
\begin{flalign}
\label{eq:dpsgd_eq1}
&\mathcal{A}(\sD) \sim \mathcal{N}(\vv, \sigma^{2}\mI). \\
\label{eq:dpsgd_eq2}
&\mathcal{A}(\sD') \sim \mathcal{N}(\vv, \sigma^{2}\mI) \;\text{with prob. $1-p$, }\notag\\
&\mathcal{A}(\sD') \sim \mathcal{N}(\vv', \sigma^{2}\mI) \;\text{with prob. $p$}.
\end{flalign}
The RDP of $\vd$ at the current step is the Rényi divergences between Equation~(\ref{eq:dpsgd_eq1}) and~(\ref{eq:dpsgd_eq2}). For a given $\sigma$ and $p$, the divergences are determined by the $L_{2}$ norm of $\vg=\vv'-\vv$. Therefore, a larger gradient would result in a larger privacy cost. If we consider all possible $\theta_{t-1}\in \mathcal{O}_{t-1}$ as required by Definition~\ref{def:definition_dp}, we would have to compute the divergences with the largest possible magnitude of $\vg$. \citet{AbadiCGMMTZ16} use the gradient clipping threshold to compute the privacy cost, which results in a worst-case privacy analysis for every example.

\subsection{Output-specific $(\varepsilon,\delta)$-Differential Privacy}
\label{subsec:output_idp}

We define output-specific individual $(\varepsilon, \delta)$-differential privacy to provide a fine-grained analysis of individual privacy.
It makes the privacy parameter $\varepsilon$ a function of the outputs and the target datapoint.

\begin{definition}
\label{def:definition_output_idp}[Output-specific individual $(\varepsilon,\delta)$-DP]
Fix a datapoint $\vd$ and a set of outcomes $\sA\subset \mathcal{O}$. Let $\sD$ be an arbitrary dataset and $\sD'=\sD\cup \{\vd\}$, an algorithm $\mathcal{A}: \mathcal{D}\rightarrow \mathcal{O}$  satisfies output-specific individual $(\varepsilon(\sA,\vd),\delta)$-DP for $\vd$ at $\sA$ if for any $\sS\subset\sA$  
\[\Pr[\mathcal{A}(\sD)\in \sS]\leq e^{\varepsilon(\sA,\vd)}\Pr[\mathcal{A}(\sD')\in \sS]+\delta,\]
\[\Pr[\mathcal{A}(\sD')\in \sS]\leq e^{\varepsilon(\sA,\vd)}\Pr[\mathcal{A}(\sD)\in \sS]+\delta.\]
\end{definition}

With $\varepsilon$ being a function of the specified set of outcomes $\sA$, we can analyze the individual privacy of models trained by DP-SGD. This is done by fixing the  models from the first $T-1$ steps $(\theta_{1},\ldots,\theta_{T-1})$, which fully specify the gradients along the training. We note that when the value of $\delta$ is larger than the probability of $\sA$, Definition~\ref{def:definition_output_idp} is trivially satisfied for any $\varepsilon$, thus it along does not provide meaningful privacy guarantee. Consequently, this definition should not guide the design of new algorithms in such scenarios. In this work, we adhere to the original implementation of DP-SGD and use Definition~\ref{def:definition_output_idp} to understand the individual privacy of DP-SGD. The resulting privacy parameters provide insights into the model's individual privacy and reflect empirical risks against membership inference attacks \citep{shokri2017membership}, as shown in Appendix~\ref{apdx:disparate_empirical_risk}.



\begin{figure}
\centering
\includegraphics[width=0.4\linewidth]{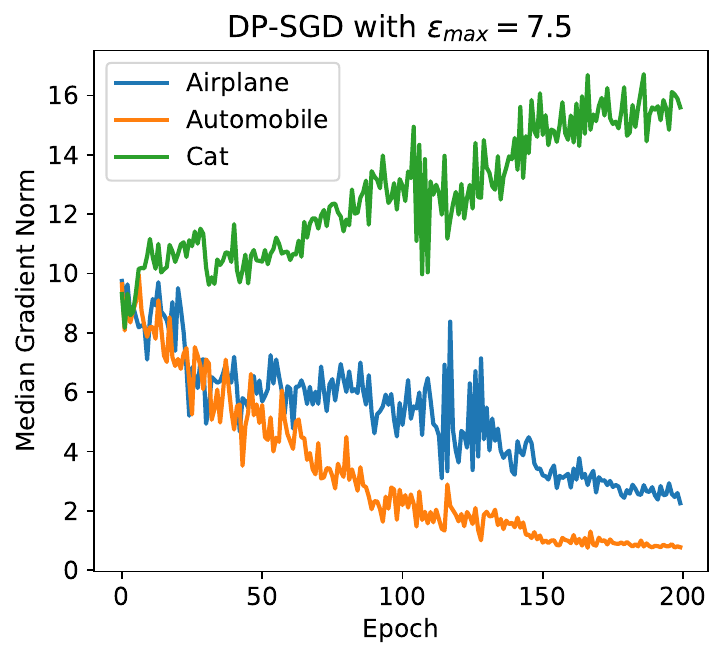}
\caption{Median of gradient norms of different classes when training a ResNet-20 model on CIFAR-10.}

\label{fig:norm_trend}
\end{figure}

\subsection{Gradients of Different Examples Vary Significantly}
\label{subsec:diverse_norms}

At each step of DP-SGD, the privacy cost of an example depends on its gradient at the current step (see Section~\ref{subsec:background} for details).
In this section, we empirically show gradients of different examples vary significantly  to demonstrate that different examples experience very different privacy costs. 
We train a ResNet-20 model with DP-SGD on CIFAR-10. 
The maximum clipping threshold is the median of gradient norms at initialization.
More implementation details are in Section~\ref{sec:exp}.
We plot the median of gradient norms of three different classes in  Figure~\ref{fig:norm_trend}. 
The gradient  norms of different classes  show  significant stratification. 
Such stratification naturally leads to different privacy costs. 
This suggests that it is meaningful to further quantify individual privacy parameters.

\section{Individual Privacy of DP-SGD}
\label{sec:algo}

Algorithm~\ref{alg:main_algo} shows the implementation of DP-SGD\footnote{Our implementation of DP-SGD follows the privacy analysis in \citet{AbadiCGMMTZ16} which uses Poisson sampling. We note that many existing implementations of DP-SGD use shuffle data instead of Poisson sampling to enforce stochasticity. Shuffle data is easier to implement but using it would create a mild discrepancy with the analysis in \citet{AbadiCGMMTZ16}. Formal privacy analysis of shuffle data requires \emph{privacy amplification by shuffling} \citep{koskela2023numerical,wang2023privacy,feldman2023stronger}.} \citep{AbadiCGMMTZ16}. Theorem~\ref{thm:privacyalgo1} gives the individual privacy analysis of DP-SGD.  Algorithm~\ref{alg:individual_privacy} gives the pseudocode  for computing individual privacy parameters. At each step, Algorithm~\ref{alg:individual_privacy} uses the (estimated) individual gradient norms to compute per-step Rényi differential privacy (RDP) for every example. It also updates the individual gradient norms and the accumulated RDP. We introduce two arguments in Algorithm~\ref{alg:individual_privacy} to reduce the computational cost of individual privacy accounting. The first one is the frequency $K$ of computing batch gradient norms  and the second one is whether to round individual gradient norms with a small constant $r$. More discussion on these two arguments could be found in Section~\ref{subsec:dp_sgd_formula} and~\ref{subsec:efficient}.

\begin{restatable}{thm}{privacymainalgo}
\label{thm:privacyalgo1}
 Let $\{\theta_{1},\ldots,\theta_{t-1}\}$ be the observed models at step $t$. Suppose we run  Algorithm~\ref{alg:main_algo} with $K=1$ and without rounding, then  Algorithm~\ref{alg:main_algo} satisfies $(o^{(i)}_{\alpha} + \frac{\log(1/\delta)}{\alpha-1}, \delta)$-output-specific individual DP for the $i_{th}$ example at $\sA=(\theta_{1},\ldots,\theta_{t-1},\mathcal{O}_{t})$, where $o^{(i)}_{\alpha}$ is the accumulated RDP  at order $\alpha$ and $\mathcal{O}_{t}$ is the range of $\mathcal{A}_{t}$.
\end{restatable}

\begin{hproof}
At step $t$, given the observed models $(\theta_{1},\ldots,\theta_{t-1})$, the composited training algorithm is 
\[\mathcal{\hat A}^{(t)}=(\mathcal{A}_{1}(\sD), \mathcal{A}_{2}(\theta_{1},\sD),\ldots,\mathcal{A}_{t}(\theta_{1},\ldots,\theta_{t-1},\sD)).\] 
 We first use the composition theorem to show the accumulated RDP is the RDP of $\mathcal{\hat A}^{(t)}$. Then we prove the RDP bound on $\mathcal{\hat A}^{(t)}$ gives an   output-specific-$(\varepsilon,\delta)$-DP bound on Algorithm~\ref{alg:main_algo}. We relegate the proof to Appendix~\ref{apdx:proof}.
\end{hproof}

\begin{algorithm}[t]

	\caption{Differentially Private SGD}
	\label{alg:main_algo}

\begin{algorithmic}
    \STATE \textbf{Input:} Clipping threshold  $C$, noise variance $\sigma^{2}$, sampling probability $p$, number of steps $T$.
	
    \medskip

    \STATE Let $\{Z^{(i)}=C\}_{i=1}^{n}$ be the estimates of individual gradient norms, initialized as $C$.
    
    \smallskip

    \STATE Let $\{\vo^{(i)}=0\}_{i=1}^{n}$ be  the accumulated individual RDP.

    \FOR{$t=0\;to\;T-1$} 

        \STATE \textsl{//Individual privacy accounting.}
        \STATE Run Algorithm~\ref{alg:individual_privacy} with $\{Z^{(i)}\}_{i=1}^{n}$ and $\{\vo^{(i)}\}_{i=1}^{n}$.

        \STATE Update $\{Z^{(i)}\}_{i=1}^{n}$ and $\{\vo^{(i)}\}_{i=1}^{n}$ with the results of Algorithm~\ref{alg:individual_privacy}.

        \medskip
        \STATE \textsl{//Run DP-SGD as usual.}
    	\STATE Sample a minibatch of gradients $\{\vg^{(I_{j})}\}_{j=1}^{|I|}$ with probability $p$ , where $I$ is the sampled indices.
         
            \STATE Clip gradients $\bar \vg^{(I_{j})} = clip(\vg^{(I_{j})},C)$.

    	\STATE Update model $\theta_{t} = \theta_{t-1} - \eta(\sum \bar \vg^{(I_{j})} + z)$, where $z\sim \mathcal{N}(0,\sigma^{2}\mI)$.
    
    \ENDFOR

 \end{algorithmic}
\end{algorithm}

\begin{remark}
Algorithm~\ref{alg:individual_privacy} does not change the worst-case privacy guarantee (Definition~\ref{def:definition_dp}) of DP-SGD because it does not modify the update rule.
\end{remark}

In Theorem~\ref{thm:privacyalgo1}, we run Algorithm~\ref{alg:individual_privacy} with $K=1$ and without rounding individual gradient norms. This configuration is computationally expensive for two reasons. Firstly, setting $K=1$ requires  computing batch gradient norms  at each SGD update. Secondly, the number of unique gradient norms is large without rounding. Each unique gradient norm corresponds to a different single-step RDP that needs to be computed numerically. In Section~\ref{subsec:dp_sgd_formula}, we give more details on  the computational challenges. In Section~\ref{subsec:efficient}, we use a larger $K$ and round individual gradient norms to provide estimations of individual privacy. This greatly improves the efficiency of Algorithm~\ref{alg:individual_privacy}.  In Section~\ref{subsec:accurate}, we show the estimates of individual privacy parameters are accurate.

\begin{algorithm}[t]

	\caption{Individual Privacy Accounting for DP-SGD}
	\label{alg:individual_privacy}

\begin{algorithmic}
    \STATE \textbf{Input:}  Individual gradient norms  $\{Z^{(i)}\}_{i=1}^{n}$, accumulated individual RDP $\{\vo^{(i)}\}_{i=1}^{n}$, frequency $K$ of updating $\{Z^{(i)}\}_{i=1}^{n}$, rounding precision $r$, current iteration $t$.

    \smallskip

    \IF{$t \bmod K =0$}
      \STATE //\textsl{Update individual sensitivity.}
        \STATE Compute batch gradient norms $\{\norm{\vg^{(i)}}_{2}\}_{i=1}^{n}$.
        \STATE Update $Z^{(i)}=\min(\norm{\vg^{(i)}}_{2}, C)$.
        \IF{use rounding}
        \STATE //\textsl{Reduce the number of different norms.}
            \STATE Update $\{Z^{(i)}=\argmin_{c\in \sC}(|c-Z^{(i)}|)\}_{i=1}^{n}$, where $\sC=\{r,2r,\ldots,C\}$ contains all possible norms.
        \ENDIF
    \ENDIF
    	
    \medskip
    
    \STATE //\textsl{Compute the current step RDP.}
    \STATE Compute  the Rényi divergences between Equation~\ref{eq:dpsgd_eq1} and~\ref{eq:dpsgd_eq2} numerically with $Z_{i}$, $p$, and $\sigma^{2}$ and store the result in $\boldsymbol \rho^{(i)}$.

      \STATE //\textsl{Update the accumulated  RDP.}
    \STATE $\vo^{(i)} = \vo^{(i)} + \boldsymbol \rho^{(i)}$.

    \RETURN $\{Z^{(i)}\}_{i=1}^{n}$, $\{\vo^{(i)}\}_{i=1}^{n}$
	
 \end{algorithmic}
\end{algorithm}

\begin{figure*}[!ht]
    \centering
  \includegraphics[width=1.0\linewidth]{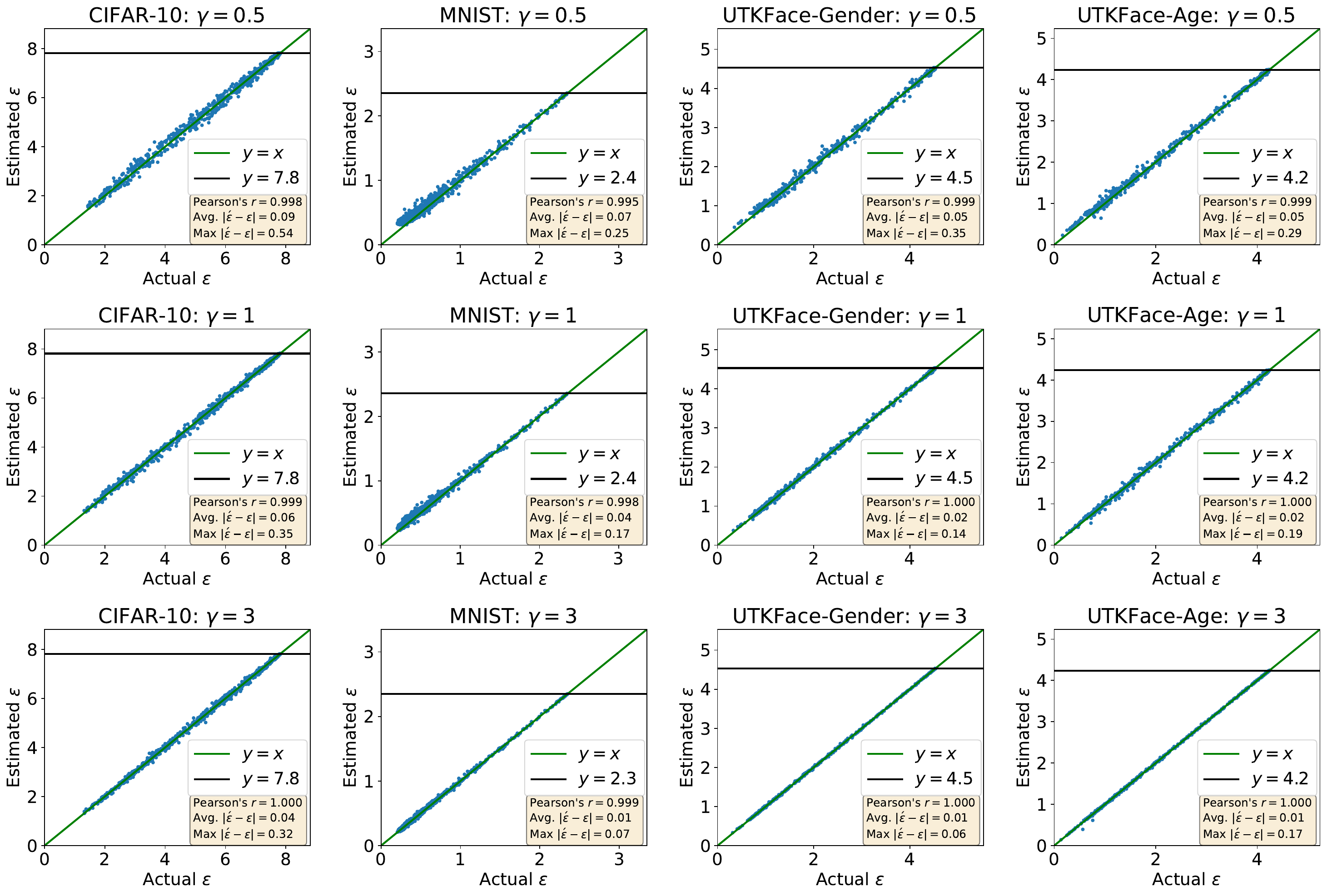}
  \caption{Privacy parameters based on estimations of individual gradient norms ($\varepsilon$) versus those based on exact ones ($\acute{\varepsilon}$). The value of $\gamma$ denotes the number of updates of full gradient norms per epoch. The horizontal line shows the worst-case privacy guarantee.
  }
  \label{fig:eps_error}
\end{figure*}

\subsection{Computational Challenges of Individual Privacy}
\label{subsec:dp_sgd_formula}

The first challenge is that computing exact privacy costs at each step requires batch gradient norms, which is impractical  when running SGD. At each update, we need the gradient of every example to compute the corresponding RDP between Equation~\ref{eq:dpsgd_eq1} and~\ref{eq:dpsgd_eq2}. The worst-case privacy analysis of DP-SGD does not have this problem because it simply assumes all examples have the maximum possible gradient, i.e., the clipping threshold.

The second challenge is that the RDP of every example at each step needs to be computed numerically. It has been shown that numerical computations are necessary to get tight bounds on the Rényi divergences between Equation~\ref{eq:dpsgd_eq1} and~\ref{eq:dpsgd_eq2} \citep{AbadiCGMMTZ16,WangBK19,MironovTZ19,GopiLW21}.
In \cite{AbadiCGMMTZ16}, only one numerical computation is required because all examples are assumed to have the worst-case privacy cost. However, when computing individual privacy parameters, the number of numerical computations is the same as the number of different gradients that could be as large as $n\times T$, where $n$ is the dataset size and $T$ is the number of iterations.

\subsection{Improving the Efficiency of Individual Privacy Accounting}
\label{subsec:efficient}

We only compute the batch gradient norms every $K$ iteration to reduce the computational overhead.  The norms are then used to estimate the privacy costs for the subsequent iterations. We note that providing estimates of privacy costs is inevitable when the computational budget is limited. This is because, by the nature of SGD, one does not have the exact individual gradient norms at every iteration.    In Appendix~\ref{apdx:individual_clip}, we explore another design choice which is to clip $\vg^{(i)}$ with $Z^{(i)}$ in Algorithm~\ref{alg:main_algo}. Although this slightly changes the implementation of DP-SGD, Algorithm~\ref{alg:individual_privacy} would return the exact privacy costs. We run experiments with this design choice and report the results in Appendix~\ref{apdx:individual_clip}. Our observations in the main text still hold in Appendix~\ref{apdx:individual_clip}.

To reduce the number of numerical computations, we round individual gradient norms with a small constant $r$.  Because the maximum clipping threshold $C$ is a constant, then, by the pigeonhole principle, there are at most $\ceil{C/r}$ different values of gradient norms, and hence there are at most $\ceil{C/r}$ different values of RDP between Equation~(\ref{eq:dpsgd_eq1}) and~(\ref{eq:dpsgd_eq2}). Note that $r$ should be small enough to avoid underestimation of RDP. We set $r=0.01C$ throughout this paper. 

We compare the computational costs with/without rounding in Table~\ref{tbl:cost}. We run the numerical method in~\cite{MironovTZ19}  once for every different value of RDP (with the default setup in the Opacus library \citep{opacus}). We run DP-SGD on CIFAR-10 for 200 epochs. The full gradient norms are updated once per epoch. All results in Table~\ref{tbl:cost} use multiprocessing with 5 cores of an AMD EPYC\textsuperscript{™} 7V13 CPU. With rounding, the overhead of computing individual privacy parameters is negligible. The computational cost without rounding is more than 7 hours.

\subsection{Estimates of Individual Privacy Are Accurate}
\label{subsec:accurate}

We run Algorithm~\ref{alg:individual_privacy} with the setup in Section~\ref{subsec:efficient} and compare the results with ground-truth values. To compute the ground-truth individual privacy, we randomly sample 1000 examples before training. During training, we compute the exact privacy costs for the same 1000 examples at every iteration. 

We compute the Pearson correlation coefficient between the estimations and the ground-truth values. We also compute the average and the worst absolute errors. We report results on MNIST, CIFAR-10, and UTKFace. Details about the experiments are in Section~\ref{sec:exp}. We plot the results in Figure~\ref{fig:eps_error}. The estimations of $\varepsilon$ are close to those ground-truth values (Pearson's $r>0.99$) even when we only update the gradient norms every two epochs ($\gamma=0.5$). Updating batch gradient norms more frequently further improves the estimation, though doing so would increase the computational overhead.

It is worth noting that the maximum clipping threshold $C$ affects the computed privacy parameters. Large  $C$ increases the variation of gradient norms (and hence the variation of privacy parameters) but leads to large noise variance while small $C$ suppresses the variation and leads to large gradient bias. Large noise variance and gradient bias are both harmful to learning \citep{chen2020understanding,song2021evading}.  In Appendix~\ref{apdx:vary_clip}, we show the influence of using different  $C$ on both accuracy and privacy.

\begin{table}
\caption{Computational costs of computing individual privacy parameters for CIFAR-10.}
\label{tbl:cost}
\centering
\small
\def\arraystretch{1.15}
\begin{tabular}{c|c|c}
\hline
 & w/ rounding & w/o rounding    \\ \hline
\# of  computations  & $1\times 10^{2}$ &  $1\times 10^{7}$ \\ \hline
Time (in seconds) &  $<3$ & $\sim 2.6\times 10^{4}$   \\ \hline
\end{tabular}
\end{table}

\subsection{What Can We Do with Individual Privacy Parameters?}
\label{subsec:data_dependent_privacy}

Individual privacy parameters depend on the private data and are thus sensitive. They can not be released publicly without care. We describe some approaches to safely make use of individual privacy parameters.

The first approach is to release $\varepsilon_{i}$ to the owner of $\vd_{i}$. This approach does not incur additional privacy cost for two reasons. First, it is safe for $\vd_{i}$ because only the rightful owner sees $\varepsilon_{i}$. Second, releasing $\varepsilon_{i}$ does not increase the privacy cost of any other example $\vd_{j}\neq\vd_{i}$. This is because computing $\varepsilon_{i}$ can be seen as a post-processing of $(\theta_{1},\ldots,\theta_{t-1})$, which is reported in a privacy-preserving manner. We prove the claim in Theorem~\ref{thm:post_processing}.

\begin{restatable}{thm}{post_processing}
\label{thm:post_processing}
Let $\mathcal{A}:\mathcal{D}\rightarrow \mathcal{O}$ be an algorithm that is $(\varepsilon_{j},\delta)$-output-specific individual DP for $\vd_{j}$ at $\sA\subset \mathcal{O}$. Let $f\left(\cdot,\vd_{i}\right):\mathcal{O}\rightarrow \mathcal{R}\times \mathcal{O}$ be a post-processing function that returns the privacy parameter of $\vd_{i}$ ($\neq \vd_{j}$) and the training trajectory. We have $f\left(\cdot,\vd_{i}\right)$ is $(\varepsilon_{j},\delta)$-output-specific DP for $\vd_{j}$ at $\sF\subset \mathcal{R} \times \mathcal{O}$ where $\sF=\{f\left(a,\vd_{i}\right): a\in \sA\}$ is all possible  post-processing results. 
\end{restatable}

\begin{proof}

First note that the construction of $f\left(\cdot,\vd_{i}\right)$ does not increase the privacy cost of $\vd_{j}$ because it is independent of $\vd_{j}$. Without loss of generality, let $\sD, \sD'\in \mathcal{D}$ be the neighboring datasets where  $\sD'={\sD\cup \{\vd_{j}\}}$. 
Let $\sS\subset \sF$ be an arbitrary event and $\sT=\{a\in \sA: f\left(a,\vd_{i}\right)\in \sS\}$. Because $f$ is a bijective function, we have
\begin{flalign}
\Pr\left[f\left(\mathcal{A}(D), \vd_{i}\right)\in \sS\right]&=\Pr\left[\mathcal{A}(D)\in \sT\right] \\
&\leq e^{\varepsilon_{j}}\Pr\left[\mathcal{A}(D')\in \sT\right] + \delta \label{eq:post_process_1}\\
& = e^{\varepsilon_{j}}\Pr\left[f\left(\mathcal{A}(D'), \vd_{i}\right)\in \sS\right] + \delta \label{eq:post_process_2},
\end{flalign}
which completes the proof. Using a bijective post-processing function is necessary for Theorem~\ref{thm:post_processing} to hold. Otherwise, there may be some $o\notin \sA$ and $a\in \sA$ that have the same processed output, which invalids the derivation from Equation~\ref{eq:post_process_1} to Equation~\ref{eq:post_process_2}. 
\end{proof}

The second approach is to privately release aggregate statistics of the population, e.g., the average or quantiles of the $\varepsilon$ values. Recent works have demonstrated such statistics can be published accurately with a minor privacy cost \cite{AndrewTMR21}.  Specifically, we privately release the average and quantiles of the $\varepsilon$ values. We report the results on CIFAR-10 and MNIST.  For releasing the average value of $\varepsilon$, we use the Gaussian Mechanism. For releasing the quantiles, we use 20 steps of batch gradient descent to solve the objective function in \cite{AndrewTMR21} with the default setup. The results are in Table~\ref{tbl:population_eps}. The released statistics are close to the actual values under $(0.1,10^{-5})$-DP. 

\begin{table*}
    \caption{Statistics of individual privacy parameters can be accurately released with minor privacy costs. The average estimation error rate is 1.13\% for MNIST and 0.91\% for CIFAR-10. The value of $\delta$ is $1\times 10^{-5}$.}
\label{tbl:population_eps}
\centering
\small
\begin{tabular}{ p{2.2cm}p{1.3cm}p{1.7cm}p{1.7cm}p{1.3cm}p{1.7cm}p{1.7cm}}
 \hline 
MNIST  &  Average  & 0.1-quantile 	& 0.3-quantile 	& Median & 0.7-quantile & 0.9-quantile \\[0.4ex]
 \hline
Non-private    &  0.686  & 0.236 &  0.318 & 0.431 & 0.697 & 1.682	\\[0.4ex]
$\varepsilon=0.1$	 & 0.681  & 0.238  & 0.317  & 0.436  & 0.708  & 1.647 	  \\[0.4ex]

 \hline
 \hline 
CIFAR-10  &  Average  & 0.1-quantile 	& 0.3-quantile 	& Median & 0.7-quantile & 0.9-quantile \\[0.4ex]
 \hline
Non-private    &   5.942 & 2.713  & 4.892 & 6.730	 & 7.692 &  7.815	\\[0.4ex]
$\varepsilon=0.1$	 &  5.939  &  2.801 & 4.876  &	6.744 & 7.672  & 7.923	  \\[0.4ex]

\hline

\end{tabular} 
\end{table*}

Finally, individual privacy parameters can also serve as a powerful tool for a trusted data curator to improve the model quality.
By analyzing the individual privacy parameters of a dataset, a trusted curator can focus on collecting more data representative of the groups that have higher privacy risks to mitigate the disparity in privacy.

\section{Individual Privacy Parameters on Different Datasets}
\label{sec:exp}

 In Section~\ref{subsec:empirical_idp}, we first show the distribution of individual privacy parameters on four  tasks. Then we study how individual privacy parameters correlate with training loss in Section~\ref{subsec:corr_loss_privacy}.   The experimental setup is as follows.

\textbf{Datasets.} We use two benchmark datasets MNIST ($n=$ 60000) and CIFAR-10 ($n=$ 50000) \citep{LeCunBBH98, Krizhevsky09} as well as the UTKFace dataset ($n\simeq$ 15000) \citep{ZhangSQ17} that contains the face images of four different races (White, $n\simeq$ 7000; Black, $n\simeq$ 3500; Asian, $n\simeq$ 2000; Indian, $n\simeq$ 2800).  We construct two classification tasks on UTKFace: predicting gender, and predicting whether the age is under $30$.\footnote{We acknowledge that predicting  gender and age from images may be problematic. Nonetheless, as facial images have previously been highlighted as a setting where machine learning has disparate accuracy on different groups, we revisit this domain through a related lens. The labels are provided by the dataset curators.}
We slightly modify the dataset between these two tasks by randomly removing a few examples to ensure each race has balanced positive and negative labels.

\textbf{Models and hyperparameters.} For CIFAR-10, we use the WRN16-4 model in \cite{DeBHSB22}, which achieves advanced performance in private setting. We follow the implementation details in \cite{DeBHSB22} expect their data augmentation method to reduce computational cost. For MNIST and UTKFace, we use ResNet20 models with batch normalization layers replaced by group normalization layers. For UTKFace, we initialize the model with weights pre-trained on  ImageNet. 

We set $C=1$ on CIFAR-10, following \cite{DeBHSB22}. For MNIST and UTKFace, we set $C$ as the median of gradient norms at initialization, following the suggestion in \cite{AbadiCGMMTZ16}. The privacy cost of using the median is not taken care of. However, the median of gradient norms could be released accurately with a small privacy cost\footnote{For instance, if we use the algorithm in \citet{AndrewTMR21} to privatize the median gradient norm of the UTKFace-Gender dataset with $(\varepsilon=0.1,\delta=1\times 10^{-5})$. The non-private median is 15.73 and the privatized median is 15.82. }. The batchsize is 4096 for CIFAR-10 and 1024 for MNIST and UTKFace. The training epoch is 300 for CIFAR-10 and 100 for MNIST and UTKFace. For a target maximum $\varepsilon$, we use the package in the Opacus library to find the corresponding noise variance \cite{YousefpourSSTPMNGBZCM21}. We update the batch gradient norms three times per epoch for all experiments in this section (the case of $\gamma=3$ in Figure~\ref{fig:eps_error}).  All experiments are run on single Tesla V100 GPUs with 32G memory. Our source code will be publicly available.

\subsection{Individual Privacy Parameters Vary Significantly}
\label{subsec:empirical_idp}

Figure~\ref{fig:eps_teaser} shows the individual privacy parameters on all datasets. The privacy parameters vary across a large range on all four tasks. On the CIFAR-10 dataset, the maximum $\varepsilon_i$  is 7.8 while  the minimum $\varepsilon_i$ is  1.0. When running gender classification on the UTKFace dataset, the maximum $\varepsilon_i$ is 4.5 while the minimum $\varepsilon_i$ is  only 0.1.

We also observe that, for easier tasks, more examples enjoy stronger privacy guarantees. For example, $\sim$35\% of examples reach the worst-case $\varepsilon$ on CIFAR-10 while only $\sim$3\% do so on MNIST.  This may be because the loss decreases quickly when the task is easy, resulting in gradient norms also decreasing and thus stronger privacy guarantees.

\subsection{Privacy Parameters and Loss Are  Positively Correlated}
\label{subsec:corr_loss_privacy}

 \begin{figure*}
    \centering
  \includegraphics[width=1.0\linewidth]{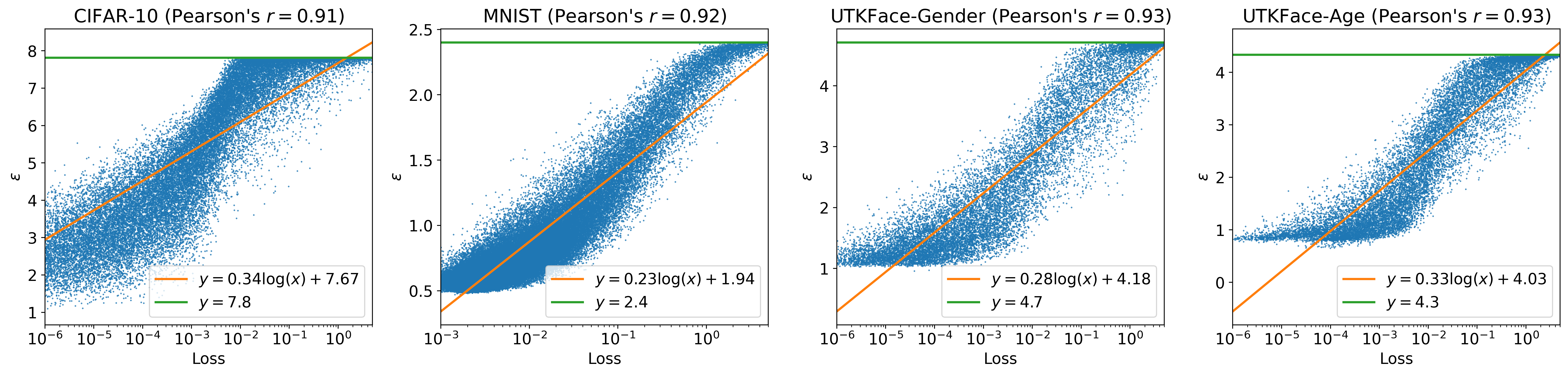}
  \caption{Privacy parameters and final training losses. Each point shows the final training loss and privacy parameter of one example. Pearson's $r$ is computed between privacy parameters and log loss values. }
  \label{fig:eps_loss_scatter}
\end{figure*}

We study how individual privacy parameters correlate with final training loss values. 
The privacy parameter of one example depends on its gradient norms along the training. 
In strongly convex optimization, the loss value of an example is reflected in the norm of its gradient. However, for non-convex deep models, there is no clear relation between the final training loss of one example and its gradient norms. Therefore, we run experiments to reveal the empirical correlation between privacy and utility.

We visualize individual privacy parameters and the final training loss values in Figure~\ref{fig:eps_loss_scatter}. The individual privacy parameters increase with  loss until they reach the maximum $\varepsilon$. To quantify the order of correlation, we further fit the points with one-dimensional logarithmic functions and compute the Pearson correlation coefficients between the privacy parameters and log loss values. The Pearson correlation coefficients are larger than $0.9$ on all datasets, showing a logarithmic correlation between the privacy parameter of a datapoint and its final training loss. In Appendix~\ref{apdx:loss_pri_corr}, we experiment with $C=5$ and $C=10$ on CIFAR-10 to study the correlation between training loss and privacy under various clipping thresholds.   The Pearson correlation coefficients are $0.89$ and $0.9$ for $C=5$ and $C=10$, respectively, suggesting that there is still a positive logarithmic correlation.

\section{Groups Are Simultaneously Underserved in Both Accuracy and Privacy}
\label{sec:fairness}

 \begin{figure}
    \centering
  \includegraphics[width=0.75\linewidth]{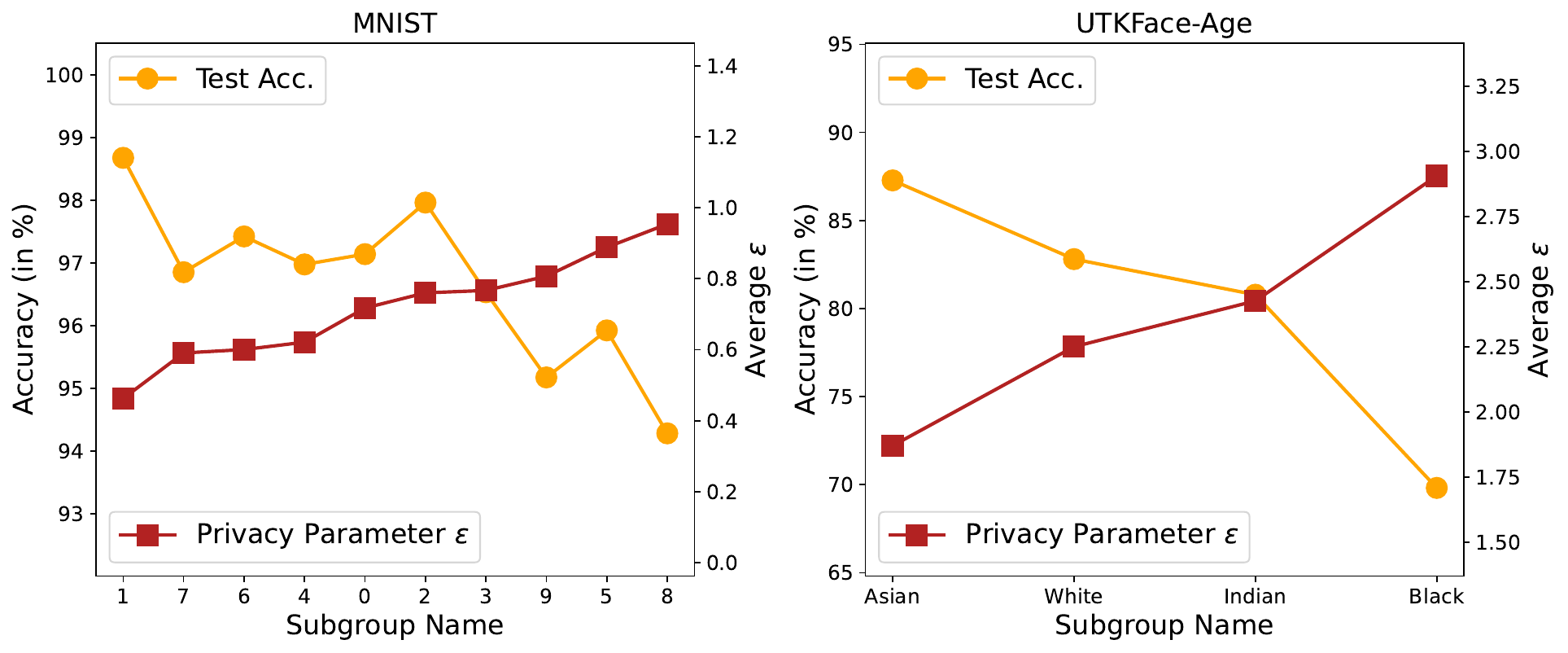}
  \caption{Accuracy and average $\varepsilon$ of different groups on MNIST and UTK-Age. Groups with worse accuracy also  have worse privacy in general. }
  \label{fig:eps_loss_corr_class_mnist}
\end{figure}

It is well-documented that the accuracy of machine learning models may be unfair for different subpopulations ~\citep{BuolamwiniG18,BagdasaryanPS19,SuriyakumarPGG21}.
Our finding demonstrates that this disparity may be simultaneous in terms of both accuracy \emph{and} privacy. 
We empirically verify this by plotting the average $\varepsilon$ and test accuracy of different groups. The experiment setup is the same as Section~\ref{sec:exp}.  For CIFAR-10 and MNIST, the groups are the data from different classes, while for UTKFace, the groups are the data from different races.

We plot the results  in Figure~\ref{fig:eps_loss_corr_class_cifar} and~\ref{fig:eps_loss_corr_class_mnist}.  The groups are sorted based on the average $\varepsilon$. The test accuracy of different groups correlates well with the average $\varepsilon$ values. Groups with worse accuracy do have worse privacy guarantees in general. On CIFAR-10, the average $\varepsilon$ of the `Cat' class (which has the worst test accuracy) is 44.2\% higher  than the average $\varepsilon$ of the `Automobile' class (which has the highest test accuracy). On UTKFace-Gender, the average $\varepsilon$ of the group with the lowest test accuracy (`Asian') is 35.1\% higher than the average $\varepsilon$ of the group with the highest accuracy (`Indian'). Similar observation also holds on other tasks. To the best of our knowledge, our work is the first to reveal this simultaneous disparity. In Appendix~\ref{apdx:disparate_empirical_risk}, we run membership inference attacks to show the disparity in privacy parameters reflects the disparity in empirical privacy risks.

\section{Conclusion}
\label{sec:conclusion}

We define output-specific individual $(\varepsilon,\delta)$-differential privacy to characterize the individual privacy guarantees of models trained by DP-SGD. 
We also design an efficient algorithm to accurately estimate the individual privacy parameters.
We use this new algorithm to examine individual privacy guarantees on several datasets. 
Significantly, we find that groups with worse utility also suffer from worse privacy. 
This new finding reveals the complex while interesting relation among utility, privacy, and fairness. 
It suggests that mitigating the utility fairness under differential privacy is more tricky than doing so in the non-private case. This is because classic methods such as upweighting underserved examples would exacerbate the disparity in privacy. 
We hope that our work sheds new light on this timely topic.

\section*{Broader Impact}
One way to utilize individual privacy parameters is to release them to corresponding users (see Section 3.4 for details). This provides a more accurate, and hence more responsible, privacy report. However, releasing individual privacy parameters may pose new challenges when a machine learning system enables data deletion, also known as machine unlearning \citep{ginart2019making,bourtoule2019machine}. Data deletions may worsen the privacy guarantees of the remaining examples \citep{carlini2022privacy}. Moreover, groups with larger $\varepsilon$ may send deletion requests more frequently than others, which could further deteriorate their privacy guarantees \citep{hashimoto2018fairness}. It's important to keep these considerations in mind when implementing both data deletion and individual privacy accounting in real-world settings.

\section*{Acknowledgments}
The authors express their gratitude to Yu-Xiang Wang and Saeed Mahloujifar for their valuable comments on an earlier version of this paper, and to the anonymous reviewers for their insightful feedback.

\bibliography{files/biblio}
\bibliographystyle{tmlr}

\appendix

\newpage

\onecolumn

\section{Proof of Theorem~\ref{thm:privacyalgo1}}
\label{apdx:proof}
\privacymainalgo*

Here we give the proof of Theorem~\ref{thm:privacyalgo1}. Let $(\mathcal{A}_{1},\ldots,\mathcal{A}_{t-1})$ be a sequence of randomized algorithms and $(\theta_{1},\ldots,\theta_{t-1})$ be some fixed outcomes, we define  
\[\mathcal{\hat A}^{(t)}(\theta_{1},\ldots,\theta_{t-1},\sD)=(\mathcal{A}_{1}(\sD), \mathcal{A}_{2}(\theta_{1},\sD),\ldots,\mathcal{A}_{t}(\theta_{1},\ldots,\theta_{t-1},\sD)).\] Noting that the individual RDP parameters  of each individual mechanism in $\mathcal{\hat A}^{(t)}(\theta_{1},\ldots,\theta_{t-1},\sD)$ are constants. Further let \[\mathcal{A}^{(t)}(\sD)=(\mathcal{A}_{1}(\sD), \mathcal{A}_{2}(\mathcal{A}_{1}(\sD), \sD),\ldots, \mathcal{A}_{t}(\mathcal{A}_{1}(\sD),\ldots,\sD))\] be the adaptive composition. In Lemma~\ref{thm:rdp_to_expost_dp},  we show an RDP bound on    $\mathcal{\hat A}^{(t)}$ gives an output-specific DP bound on  $\mathcal{A}^{(t)}$.  We comment that the individual RDP parameters  of each individual mechanism in $\mathcal{A}^{(t)}(\sD)$ are random variables.   The composition of random privacy parameters  requires additional care because the standard composition theorem requires the privacy parameters to be constants \citep{FeldmanZ21,Lcuyer21,WhitehouseRRW22}.

\begin{lemma}
\label{thm:rdp_to_expost_dp} 
Let $\sA=(\theta_{1}, \ldots,\theta_{t-1},\mathcal{O}_{t})\subset \mathcal{O}^{(t)}$ where $\theta_{1},\ldots,\theta_{t-1}$ are some arbitrary fixed outcomes and $\mathcal{O}^{(t)}$ is the domain of $\mathcal{A}^{(t)}(\sD)$ and $\mathcal{\hat A}^{(t)}(\sD)$. If $\mathcal{\hat A}^{(t)}(\cdot)$ satisfies $o_{\alpha}$ RDP at order $\alpha$, then  $\mathcal{A}^{(t)}(\sD)$ satisfies $(o_{\alpha} + \frac{\log(1/\delta)}{\alpha-1}, \delta)$-output-specific differential privacy at $\sA$.
\end{lemma}

\begin{proof}
For a given outcome $\theta^{(t)}=(\theta_{1}, \theta_{2},\ldots,\theta_{t-1},\theta_{t}) \in \sA$, we have $\mathbb{P}\left[\mathcal{A}^{(t)}(\sD)=\theta^{(t)}\right]=$
\begin{flalign}
\label{eq:proof_privacy_eq1}
&\mathbb{P}\left[\mathcal{A}^{(t-1)}(\sD)=\theta^{(t-1)}\right]\mathbb{P}\left[\mathcal{A}_{t}(\mathcal{A}_{1}(\sD),\ldots,\mathcal{A}_{t-1}(\sD),\sD)=\theta_{t}|\mathcal{A}^{(t-1)}(\sD)=\theta^{(t-1)}\right],\\
&=\mathbb{P}\left[\mathcal{A}^{(t-1)}(\sD)=\theta^{(t-1)}\right]\mathbb{P}\left[\mathcal{A}_{t}(\theta_{1},\ldots,\theta_{t-1},\sD)=\theta_{t}\right],
\end{flalign}

by the product rule of conditional probability. Apply the product rule  recurrently on $\mathbb{P}\left[\mathcal{A}^{(t-1)}(\sD)=\theta^{(t-1)}\right]$, we have $\mathbb{P}\left[\mathcal{A}^{(t)}(\sD)=\theta^{(t)}\right]=$

\begin{flalign}
\label{eq:proof_privacy_eq2}
&\mathbb{P}\left[\mathcal{A}^{(t-2)}(\sD)=\theta^{(t-2)}\right]\mathbb{P}\left[\mathcal{A}_{t-1}(\theta_{1},\ldots,\theta_{t-2},\sD)=\theta_{t-1}\right]\mathbb{P}\left[\mathcal{A}_{t}(\theta_{1},\ldots,\theta_{t-1},\sD)=\theta_{t}\right],\\
&=\mathbb{P}\left[\mathcal{A}_{1}(\sD)=\theta_{1}\right]\mathbb{P}\left[\mathcal{A}_{2}(\theta_{1},\sD)=\theta_{2}\right]\ldots \mathbb{P}\left[\mathcal{A}_{t}(\theta_{1},\ldots,\theta_{t-1},\sD)=\theta_{t}\right],\\
&=\mathbb{P}\left[\mathcal{\hat A}^{(t)}(\theta_{1},\ldots,\theta_{t-1},\sD)=\theta^{(t)}\right].
\end{flalign}

In words, $\mathcal{A}^{(t)}$ and $\mathcal{\hat A}^{(t)}$ are identical in $\sA$. Therefore, $\mathcal{A}^{(t)}$ satisfies $(\varepsilon,\delta)$-DP at any $\sS\subset\sA$ if  $\mathcal{\hat A}^{(t)}$ satisfies $(\varepsilon,\delta)$-DP. Converting the RDP bound on $\mathcal{\hat A}^{(t)}(\sD)$ into a $(\varepsilon,\delta)$-DP bound with Lemma~\ref{lma:rdp_to_dp} then completes the proof.
\begin{lemma} [Conversion from RDP to $(\varepsilon,\delta)$-DP \cite{Mironov17}]
\label{lma:rdp_to_dp}
If $\mathcal{A}$ satisfies $(\alpha,\rho)$-RDP, then $\mathcal{A}$ satisfies $(\rho+\frac{\log(1/\delta)}{\alpha-1}, \delta)$-DP for all $0<\delta<1$.
\end{lemma}
\end{proof}

\section{Individual Privacy Accounting with Individual Clipping}

We set $K>1$ in Algorithm~\ref{alg:individual_privacy} to reduce the computational cost of individual privacy accounting (see Section~\ref{sec:algo} for details). In this case, the computed privacy costs are estimates of the exact ones. In Section~\ref{subsec:accurate} we demonstrate the estimates are accurate. In this section, we give another design choice that slightly modifies the original DP-SGD to give exact privacy accounting. More specifically, we clip the individual gradients with the estimates of gradient norms $\{Z^{(i)}\}$ from Algorithm~\ref{alg:individual_privacy}. We refer to this design choice as \emph{individual clipping}. We give the implementation in Algorithm~\ref{alg:main_algo_withic} and highlight the changes in bold font. We run experiments with individual clipping and report the results in Appendix~\ref{apdx:clip_and_acc} and~\ref{apdx:individual_clip}. The experimental setup is the same as that in Section~\ref{sec:exp}.

\begin{algorithm}[t]

	\caption{Differentially Private SGD with Individual Clipping}
	\label{alg:main_algo_withic}

\begin{algorithmic}
    \STATE{\bfseries Input:} Clipping threshold  $C$, noise variance $\sigma^{2}$, sampling probability $p$, number of steps $T$.
	
    \medskip

    \FOR{$t=0\;to\;T-1$} 
	    
            \STATE \textbf{Call Algorithm~\ref{alg:individual_privacy} for individual privacy accounting and get estimates of individual gradient norms $\{Z^{(i)}\}_{i=1}^{n}$.}
    	
    	\STATE Sample a minibatch of gradients $\{\vg^{(I_{j})}\}_{j=1}^{|I|}$ with probability $p$ , where $I$ is the sampled indices.
         
            \STATE \textbf{Clip gradients $\bar \vg^{(I_{j})} = clip(\vg^{(I_{j})},Z^{(I_{j})})$.}

    	\STATE Update model $\theta_{t} = \theta_{t-1} - \eta(\sum \bar \vg^{(I_{j})} + z)$, where $z\sim \mathcal{N}(0,\sigma^{2}\mI)$.
    
    \ENDFOR

 \end{algorithmic}
\end{algorithm}

\subsection{Individual Clipping Does Not Affect Accuracy}
\label{apdx:clip_and_acc}

Algorithm~\ref{alg:main_algo_withic} uses individual clipping thresholds to ensure the computed privacy parameters are strict privacy guarantees.  If the clipping thresholds are close to the actual gradient norms, then the clipped results are close to those of using a single maximum clipping threshold. However, if the estimations of gradient norms are not accurate, individual thresholds would clip more signal than using a single maximum threshold.

\begin{table}[h]
    \caption{Comparison between the test accuracy of using individual clipping thresholds and that of using a single maximum clipping threshold. The maximum $\varepsilon$ is $7.8$ for CIFAR-10 and 2.4 for MNIST. }
\label{tbl:compare_acc}
\centering
\begin{tabular}{ p{2.25cm}p{2.2cm}p{2.2cm}}
 \hline 
               & CIFAR-10 	& MNIST  \\[0.4ex]
 \hline
Individual    &   	74.0 ($\pm$0.19)   & 		97.17 ($\pm$0.12)   	\\[0.4ex]
Maximum       & 	 74.1 ($\pm$0.24)	 &    97.26 ($\pm$0.11) 		 \\[0.4ex]

 \hline

\end{tabular} 
\end{table}

We compare the accuracy of two different clipping methods in Table~\ref{tbl:compare_acc}. The individual clipping thresholds are updated once per epoch. We repeat the experiment four times with different random seeds.   The results suggest that using individual clipping thresholds in Algorithm~\ref{alg:main_algo} has a negligible effect on accuracy.

\subsection{Individual Clipping Does Not Change The Observations}
\label{apdx:individual_clip}

Here we show running DP-SGD with individual clipping does not change our observations in Section~\ref{sec:exp} and~\ref{sec:fairness}.

\begin{figure*}
    \centering
  \includegraphics[width=1.0\linewidth]{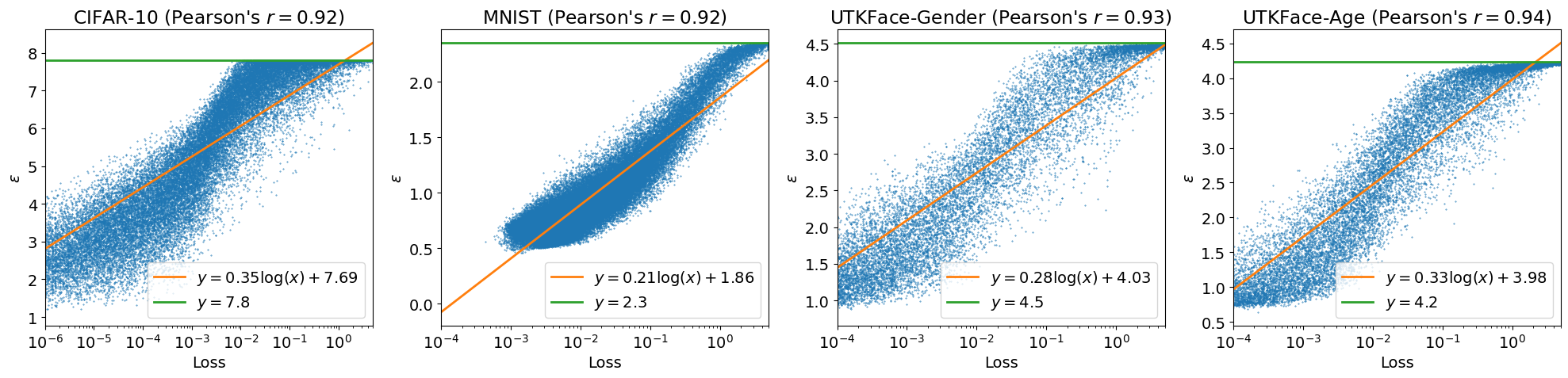}
  \caption{Privacy parameters and final training losses. The experiments are run with individual clipping (Algorithm~\ref{alg:main_algo_withic}). The Pearson correlation coefficient is computed between privacy parameters and log losses.
  }
  \label{fig:eps_loss_scatter_ic}
\end{figure*}

\paragraph{Privacy parameters have a strong correlation with individual training loss.} In Figure~\ref{fig:eps_loss_scatter_ic}, we show privacy parameters computed with individual clipping are still positively correlated with training losses. The Pearson correlation coefficient between privacy parameters and log losses is larger than $0.9$ for all datasets.

\begin{figure}
    \centering
  \includegraphics[width=0.7\linewidth]{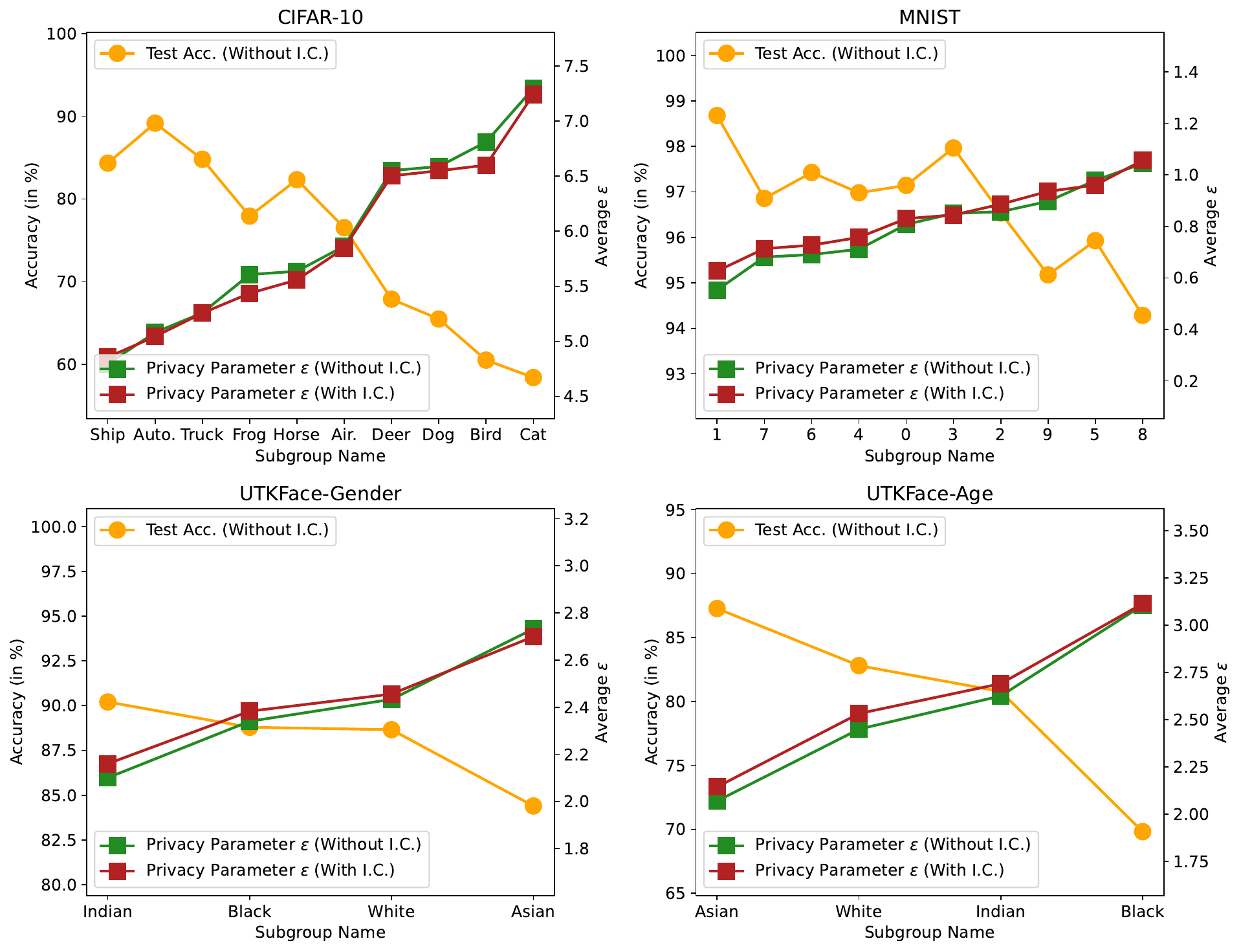}
  \caption{Test accuracy and privacy parameters computed with/without individual clipping (I.C.). Groups with worse test accuracy also have worse privacy in general.
  }
  \label{fig:eps_loss_corr_class_ic}
\end{figure}

\paragraph{Groups are simultaneously underserved in both accuracy and privacy} We show our observation in Section~\ref{sec:fairness}, i.e., low-accuracy groups  have worse privacy parameters, still holds in Figure~\ref{fig:eps_loss_corr_class_ic}. We also  make a direct comparison with privacy parameters computed without individual clipping. We find that privacy parameters computed with individual clipping are close to those computed without individual clipping.  We also find that the order of groups, sorted by the average $\varepsilon$, is exactly the same for both cases.

 \begin{figure}
    \centering
  \includegraphics[width=0.7\linewidth]{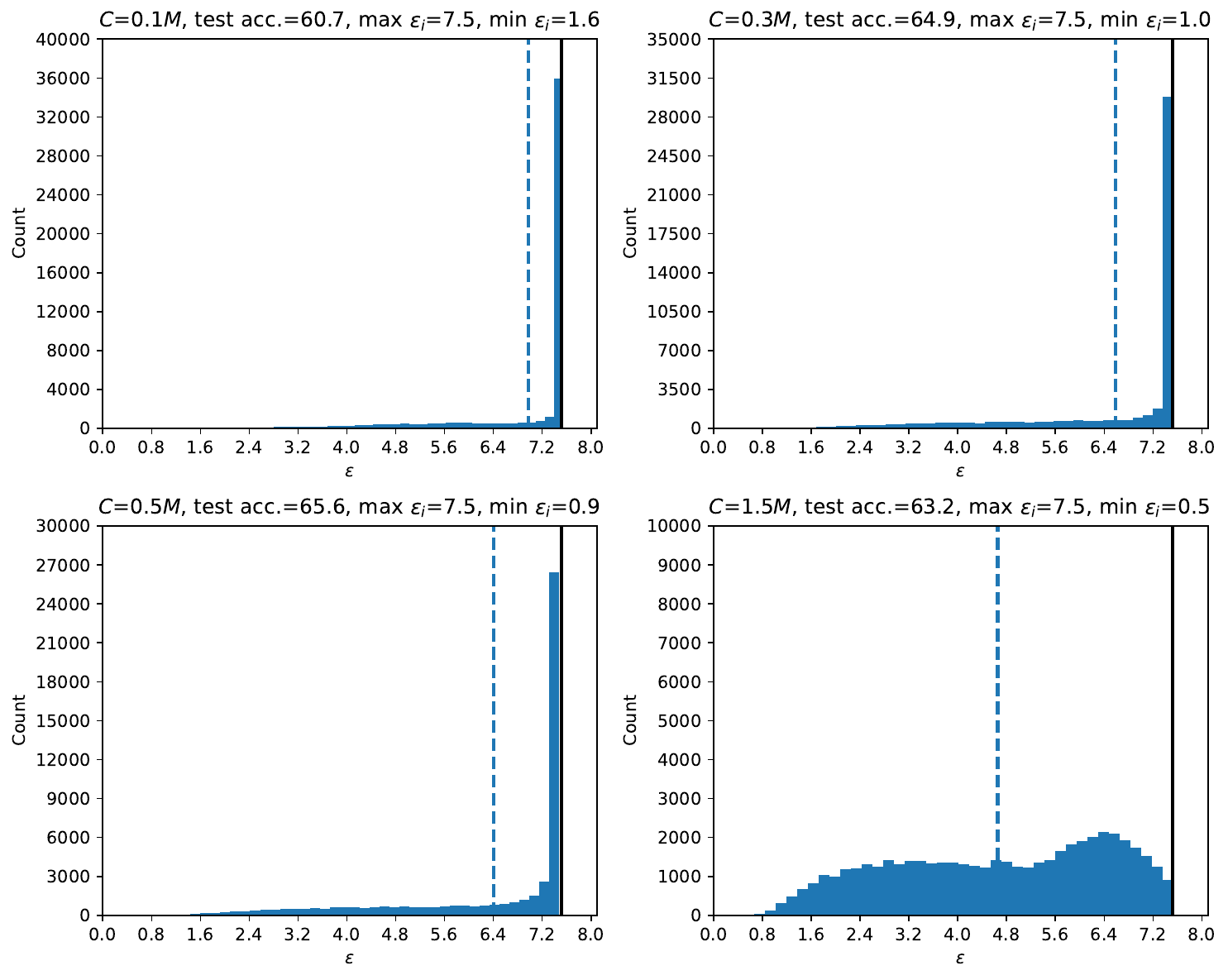}
  \caption{Distributions of individual privacy parameters on CIFAR-10 with different maximum clipping thresholds.  The dashed line indicates the average of privacy parameters. }
  \label{fig:vary_clip_hist}
\end{figure}

 \begin{figure}
    \centering
  \includegraphics[width=0.7\linewidth]{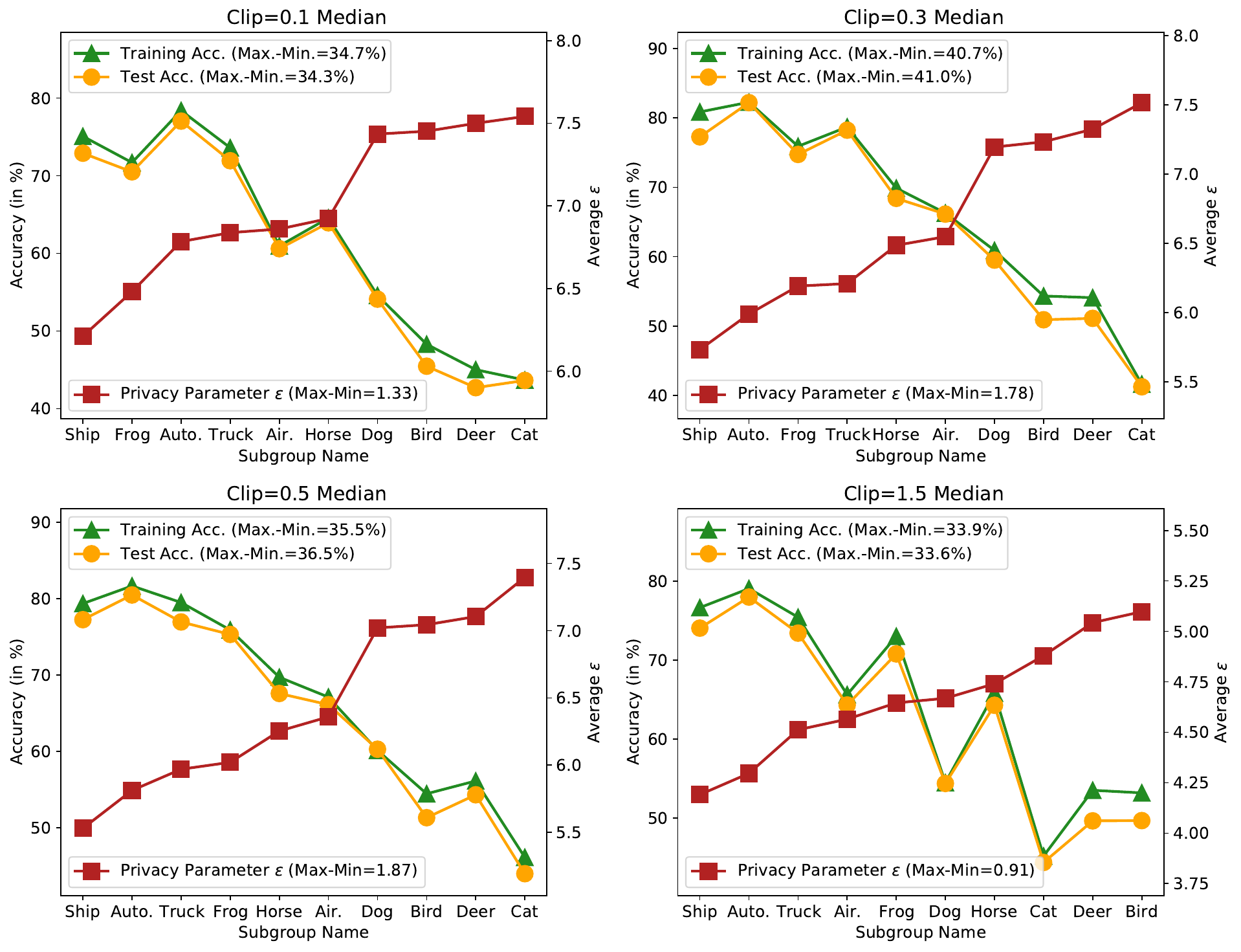}
  \caption{Accuracy and average $\varepsilon$ of different groups on CIFAR-10 with different maximum clipping thresholds.  }
  \label{fig:vary_clip_corr}
\end{figure}

\section{The Influence of Different Maximum Clipping Thresholds}
\label{apdx:vary_clip}

The value of the maximum clipping threshold $C$ would affect individual privacy parameters. A large value of $C$ would increase the stratification in gradient norms but also increase the noise variance for a fixed privacy budget. A small value of $C$ would suppress the stratification but also increase the gradient bias.  Here we run experiments with different values of $C$ on CIFAR-10. We use a small ResNet20 model in \cite{HeZRS16}, which only has $\sim$0.2M parameters, to reduce the computation cost. All batch normalization layers are replaced with group normalization layers. Let $M$ be the median of gradient norms at initialization, we choose $C$ from the list $[0.1M, 0.3M, 0.5M, 1.5M]$.

The histograms of individual privacy parameters are in Figure~\ref{fig:vary_clip_hist}. In terms of accuracy, using clipping thresholds near the median gives better test accuracy. In terms of privacy, using smaller clipping thresholds increases privacy parameters in general. The number of datapoints that reaches the worst privacy decreases with  the value of $C$. When $C=0.1M$, nearly 70\% datapoints reach the worst privacy parameter while only $\sim$2\% datapoints reach the worst parameter when $C=1.5M$. 

The correlation between accuracy and privacy is in Figure~\ref{fig:vary_clip_corr}. The disparity in average $\varepsilon$ is clear for all choices of $C$. Another important observation is that when decreasing $C$, the privacy parameters of underserved groups increase quicker than other groups. When changing $C=1.5M$ to $0.5M$, the average $\varepsilon$ of `Cat' increases from 4.8 to 7.4, almost reaching the worst-case bound. In comparison, the increment in $\varepsilon$ of the `Ship' class is only 1.3 (from 4.2 to 5.5).

\section{Privacy Parameters Reflect Empirical Privacy Risks in Non-Private Learning}
\label{apdx:disparate_empirical_risk}
We run membership inference (MI) attacks to verify whether examples with larger privacy parameters have higher privacy risks in practice. We use a simple loss-threshold  attack that predicts an example is a member if its loss value is smaller than a prespecified threshold \citep{sablayrolles2019white}.  Previous works show that even large privacy parameters are sufficient to defend against such attacks \citep{carlini2019secret,yu2021large}. In order to better observe the difference in privacy risks, we also include models trained without differential privacy as target models. For each data subgroup, we use its whole test set and a random subset of the training set so the numbers of training and test loss values are balanced. We further split the data into two subsets evenly to find the optimal threshold on one and report the success rate on another.

 \begin{figure} [h]
    \centering
  \includegraphics[width=0.9\linewidth]{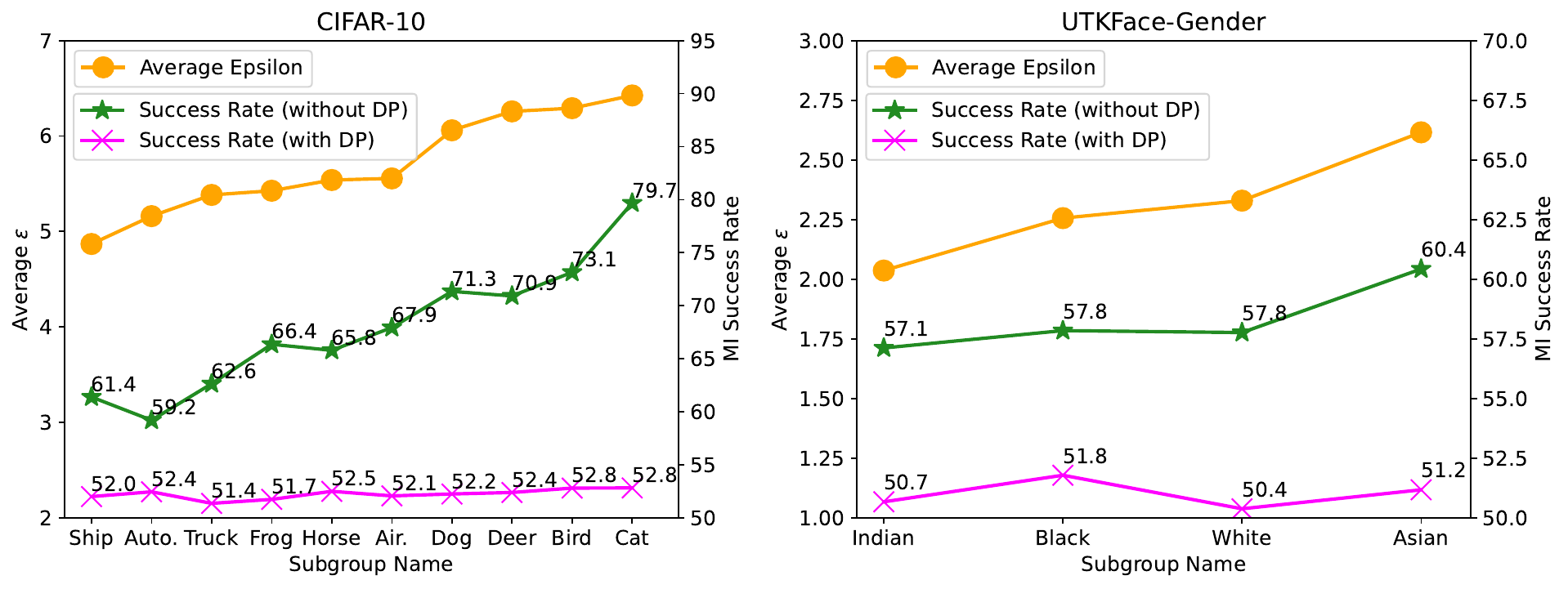}
  \caption{Average  $\epsilon$ and membership inference success rates on different subgroups. }
  \label{fig:eps_mi_corr}
\end{figure}

The results on CIFAR-10 and UTKFace-Gender are in Figure~\ref{fig:eps_mi_corr}.  The subgroups are sorted based on their average $\epsilon$. When the models are trained with DP, all attack success rates are close to random guessing (50\%).  Although the attack we use can not show the disparity in this case, we note that  there are more powerful attacks whose success rates are closer to the lower bound that DP offers \citep{nasr2021adversary}. On the other hand, the difference in privacy risks is clear when models are trained without DP. On CIFAR-10, the MI success rate is 79.7\% on the Cat class (which has the worst average $\epsilon$ when trained with DP) while is only 61.4\% on the Ship class (which has the best average $\epsilon$). These results suggest that the $\epsilon$ values reflect empirical privacy risks which could vary  significantly in different subgroups.

\section{The  Correlation Between Privacy Parameters and Loss Holds under Different Clipping Thresholds}
\label{apdx:loss_pri_corr}

In Section~\ref{subsec:corr_loss_privacy}, we show that there is a positive logarithmic correlation between privacy parameters and training loss. In this section, we run experiments on CIFAR-10 with $C=5$ and $C=10$ to show the correlation still holds under different clipping thresholds. The experiment setup, except the value of $C$, is the same as Section~\ref{subsec:corr_loss_privacy}. The results are in Figure~\ref{fig:loss_pri_corr_apdx}. Although changing the clipping threshold changes the slope and intercept, the logarithmic correlation is still strong.  The Pearson correlation coefficients are $0.89$ and $0.9$ for $C=5$ and $C=10$, respectively.

 \begin{figure} [h]
    \centering
  \includegraphics[width=0.75\linewidth]{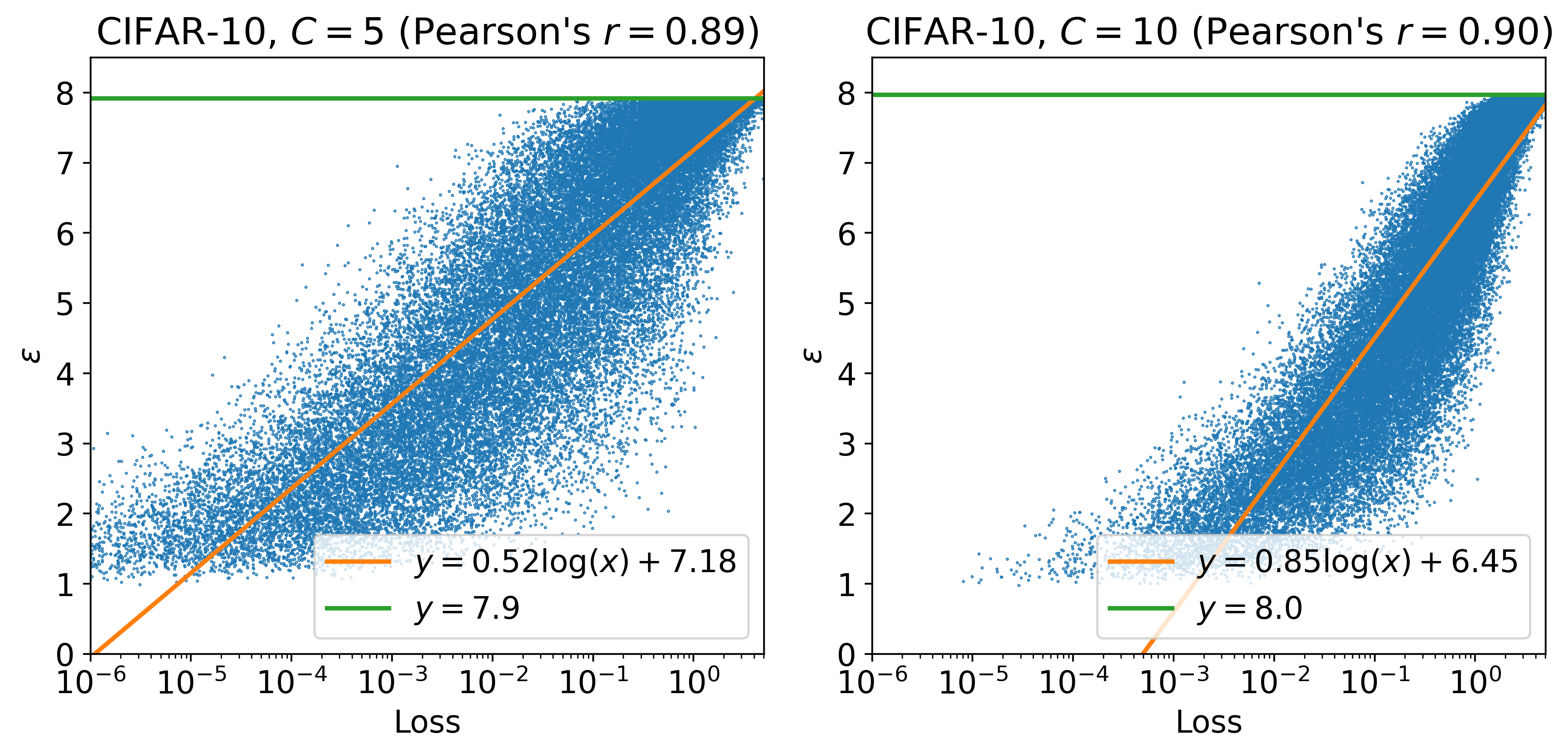}
  \caption{Privacy parameters and final training losses. Each point shows the final training loss and privacy parameter of one example. Pearson's $r$ is computed between privacy parameters and log loss values.}
  \label{fig:loss_pri_corr_apdx}
\end{figure}

\section{Individual Privacy Accounting in More Settings}
\label{apdx:vary_loss_arch}

In this section, we study individual privacy accounting in more experimental settings. The dataset in this section is CIFAR-10. We first study the influence of loss function. We replace the cross-entropy loss with the multi-class hinge loss. Other settings are the same as those in Section~\ref{sec:exp}. The results are in Figure~\ref{fig:margin_loss}. We also study the influence of model architectures. We replace WRN16-4 with the two-layer convolutional neural networks in \citet{papernot2020tempered} and still use the cross-entropy loss. The learning rate is set as $1.0$ and other settings are the same as those in Section~\ref{sec:exp}. The results are in Figure~\ref{fig:small_cnn}.

Our main observations in the main text still hold in the new settings. The Pearson's correlation coefficients between the estimated and actual privacy parameters are larger than $0.99$ in both cases. Examples that are underserved by the model also suffer from higher privacy costs. Although the main observations do not change, we observe some differences in the privacy parameters. When using the two-layer neural network instead of WRN16-4,  the privacy parameters become larger. For instance, the average $\varepsilon$ of Cat increases from 7.3 to 7.7. 

\begin{figure} 
  \centering
  \begin{subfigure}[t]{.4\linewidth}
    \centering\includegraphics[width=1.0\linewidth]{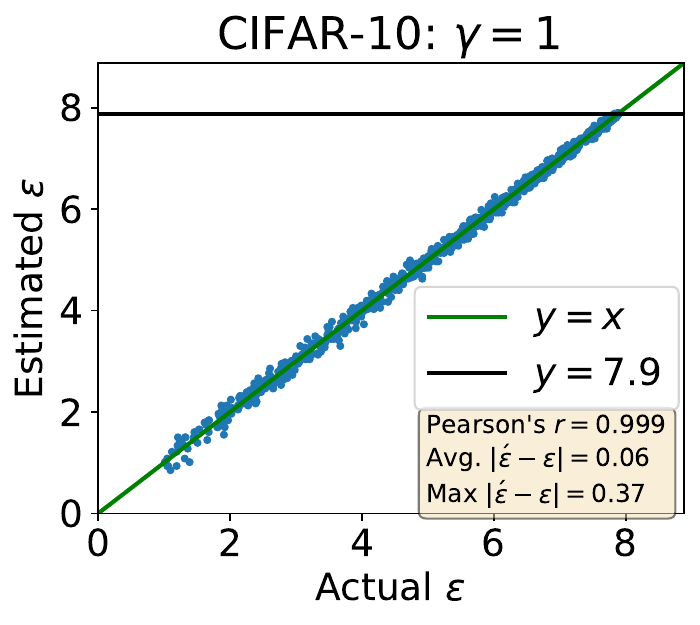}
    \caption{Estimated $\varepsilon$ versus exact $\varepsilon$.}
  \end{subfigure}
  \begin{subfigure}[t]{.4\linewidth}
    \centering\includegraphics[width=1.0\linewidth]{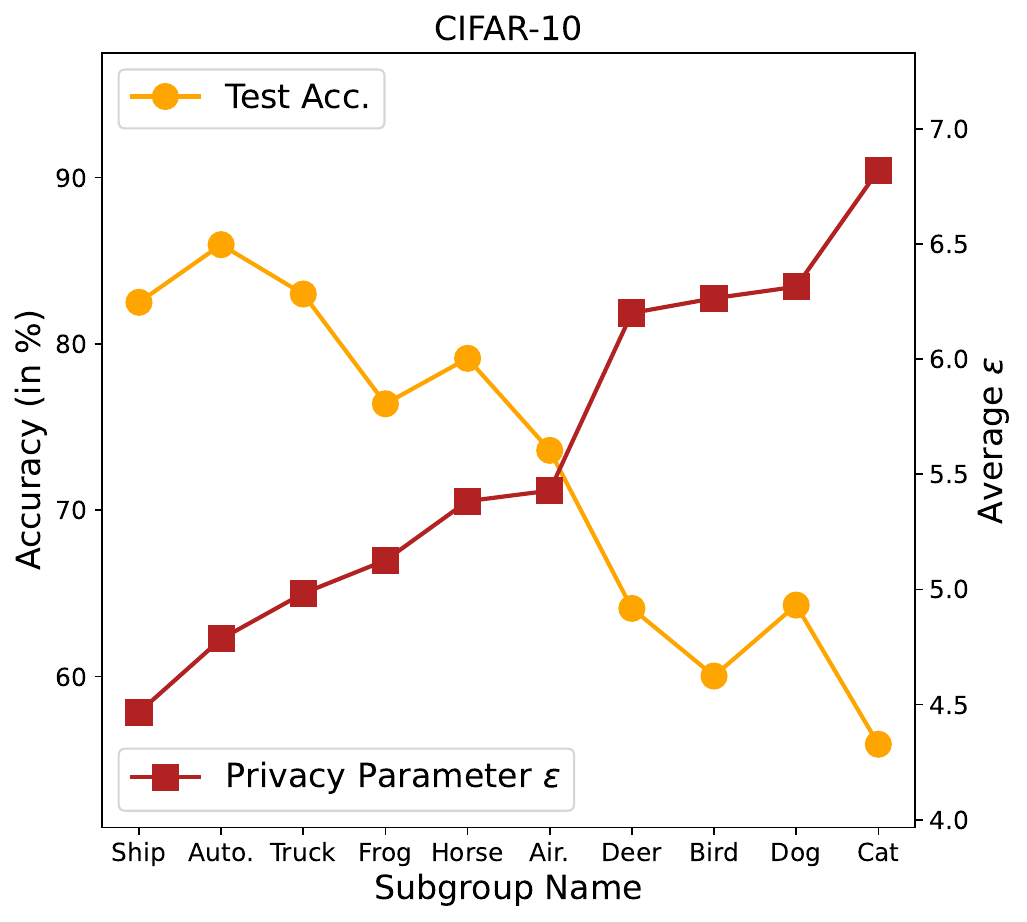}
    \caption{Accuracy and group-wise $\varepsilon$.}
  \end{subfigure}
  \caption{The results of using multi-class hinge loss instead of cross-entropy loss. 
}
  \label{fig:margin_loss}
\end{figure}

\begin{figure} 
  \centering
  \begin{subfigure}[t]{.4\linewidth}
    \centering\includegraphics[width=1.0\linewidth]{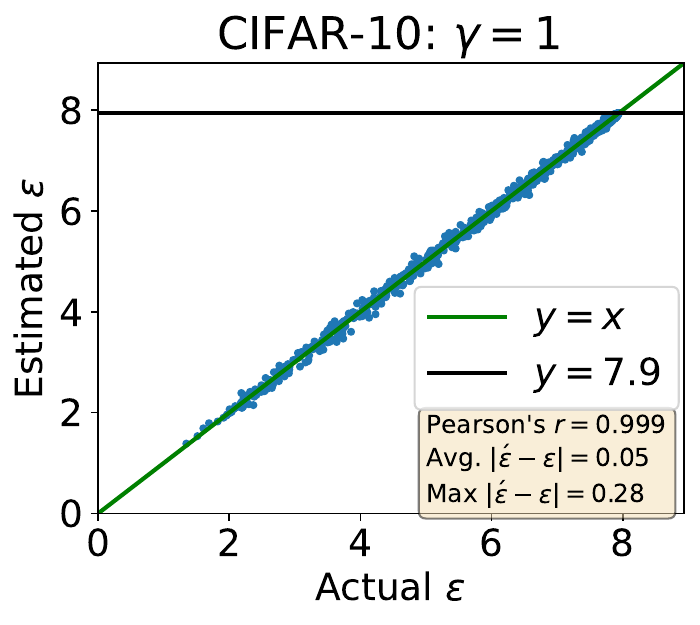}
    \caption{Estimated $\varepsilon$ versus exact $\varepsilon$.}
  \end{subfigure}
  \begin{subfigure}[t]{.4\linewidth}
    \centering\includegraphics[width=1.0\linewidth]{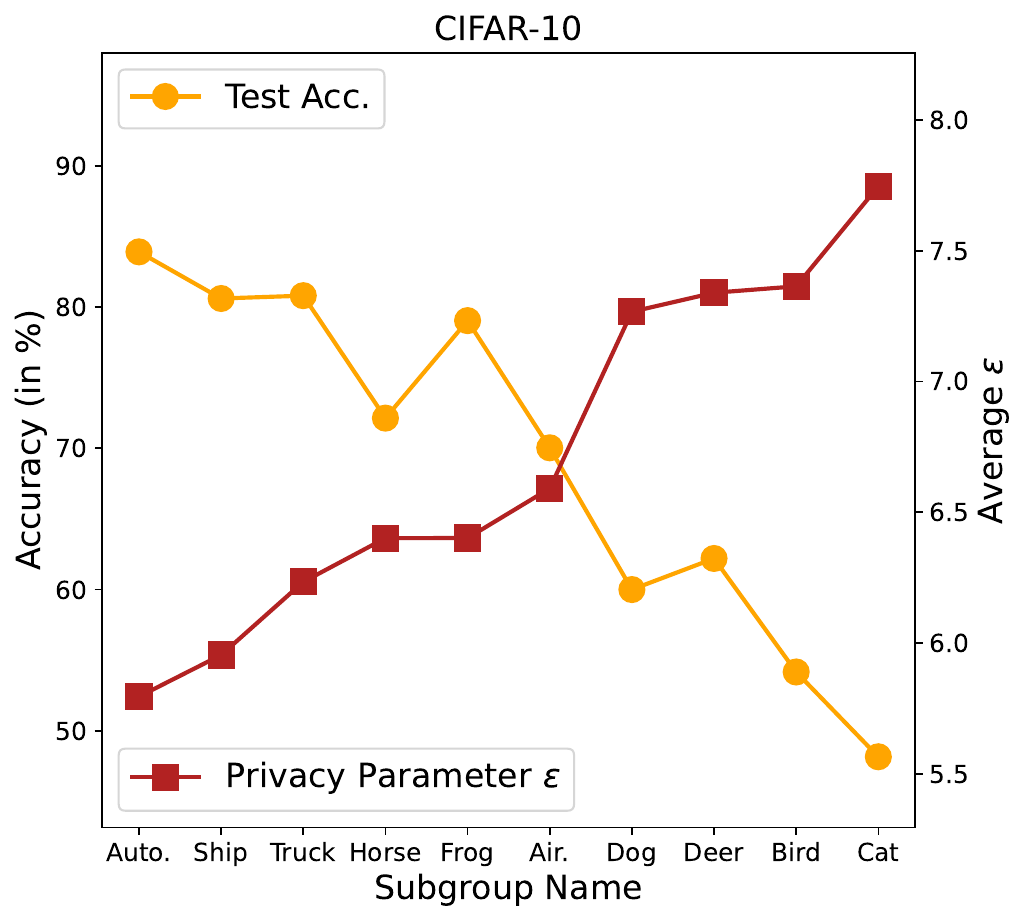}
    \caption{Accuracy and group-wise $\varepsilon$.}
  \end{subfigure}
  \caption{The results of using the small convolutional neural network in \citet{papernot2020tempered}. 
}
  \label{fig:small_cnn}
\end{figure}

\end{document}